\documentclass[5p,times]{elsarticle}

\usepackage{lineno,hyperref}
\usepackage{hyperref}
\usepackage{soul}
\usepackage{xcolor}
\usepackage{url}
\usepackage[small]{caption}
\usepackage{graphicx}
\usepackage{color}
\usepackage{amsmath}
\usepackage{amsthm}
\usepackage{amssymb}  % assumes amsmath package installed
\usepackage{mathtools}
\usepackage{booktabs}
\usepackage{subcaption}
\usepackage{multirow}
\usepackage[ruled,vlined,linesnumbered]{algorithm2e}
\usepackage{tabulary}
\usepackage[normalem]{ulem}
\usepackage{soul}
\definecolor{ashgrey}{rgb}{0.7, 0.75, 0.71}
\sethlcolor{ashgrey}
\usepackage[framemethod=tikz,hidealllines=true,backgroundcolor=blue!10]{mdframed}
\usepackage[most]{tcolorbox}

\tcbset{
    frame code={}
    center title,
    left=0pt,
    right=0pt,
    top=0pt,
    bottom=0pt,
    colback=gray!30,
    colframe=white,
    width=\dimexpr\textwidth\relax,
    enlarge left by=0mm,
    boxsep=5pt,
    arc=0pt,outer arc=0pt,
    }

\newtheorem{theorem}{Theorem}

\newtheorem{lemma}[theorem]{Lemma}

\DeclareMathOperator*{\argmax}{arg\,max}
\newcommand{\adaBracket}[1]{\left(#1 \right)}

\newcommand{\mynorm}[2]{\left|\left| #1 \right|\right|_{#2}}

\journal{Neural Networks}

\begin{document}

\begin{frontmatter}

\title{Cautious Policy Programming: Exploiting KL Regularization \\ for Monotonic Policy Improvement in Reinforcement Learning}

%% Group authors per affiliation:
\author{Lingwei Zhu\corref{correspondingauthor}}
\ead{lingwei.andrew.zhu@gmail.com}
\author{Toshinori Kitamura, Takamitsu Matsubara}
\address{Graduate School of Science and Technology, Nara Institute of Science and Technology, Ikoma, Nara, Japan}
\cortext[correspondingauthor]{Corresponding author}

%% or include affiliations in footnotes:

%\ead[url]{www.elsevier.com}

%\author[mysecondaryaddress]{Global Customer Service\corref{mycorrespondingauthor}}

%\address[mymainaddress]{1600 John F Kennedy Boulevard, Philadelphia}
%\address[mysecondaryaddress]{360 Park Avenue South, New York}

\begin{abstract}
In this paper, we propose cautious policy programming (CPP), a novel value-based reinforcement learning (RL) algorithm that can ensure monotonic policy improvement during learning. Based on the nature of entropy-regularized RL, we derive a new entropy-regularization-aware lower bound of policy improvement that depends on the expected policy advantage function but not on state-action-space-wise maximization as in prior work. 
CPP leverages this lower bound as a criterion for adjusting the degree of a policy update for alleviating policy oscillation. 
Different from similar algorithms that are mostly theory-oriented, we also propose a novel interpolation scheme that makes CPP better scale in high dimensional control problems.
We demonstrate that the proposed algorithm can trade off performance and stability in both didactic classic control problems and challenging high-dimensional Atari games.
\end{abstract}


\end{frontmatter}

%\linenumbers

\section{Introduction}

Reinforcement learning (RL) has recently achieved impressive successes in fields such as robotic manipulation~\cite{openai2019solving} and video game playing~\cite{mnih2015human}. However, compared with supervised learning that has a wide range of practical applications, RL applications have primarily been limited to game playing or lab robotics. A crucial reason for such limitation is the lack of guarantee that the performance of RL policies will improve monotonically; they often oscillate during policy updates. As such, deploying such updated policies without examining their reliability might bring severe consequences in real-world scenarios, e.g., crashing a self-driving car.

Dynamic programming (DP)~\cite{Bertsekas2005} offers a well-studied framework under which strict policy improvement is possible: with a known state transition model, reward function, and exact computation, monotonic improvement is ensured and convergence is guaranteed within a finite number of iterations~\cite{Ye2011}. However, in practice an accurate model of the environment is rarely available. In situations where either model knowledge is absent or the DP value functions cannot be explicitly computed, approximate DP and corresponding RL methods are to be considered. However, approximation introduces unavoidable update and Monte-Carlo sampling errors, and possibly restricts the policy space in which the policy is updated, leading to the \emph{policy oscillation} phenomenon~\cite{Bertsekas2011,wagner2011}, whereby the updated policy performs worse than pre-update policies during intermediate stages of learning. Inferior updated policies resulting from policy oscillation could pose a physical threat to real-world RL applications. Further, as value-based methods are widely employed in the state-of-the-art RL algorithms~\cite{haarnoja-SAC2018}, addressing the problem of policy oscillation becomes important in its own right.

%To prevent policy oscillation, the capability of evaluating policy improvement without deploying updated policies in the environment is crucial. Since such a quantity is difficult to obtain, 

Previous studies~\cite{Kakade02,pirotta13} attempt to address this issue by optimizing lower bounds of policy improvement: the classic conservative policy iteration (CPI) \cite{Kakade02} algorithm states that, if the new policy is linearly interpolated by the greedy policy and the baseline policy, non-negative lower bound on the policy improvement can be defined. 
Since this lower bound is a negative quadratic function in the interpolation coefficient, one can solve for the maximizing coefficient to obtain maximum improvement at every update.
CPI opened the door of monotonic improvement algorithms and the concept of linear interpolation can be regarded as performing regularization in the \emph{stochastic policy space} to reduce greediness.
Such regularization is theoretically sound as it has been proved to converge to global optimum \cite{Scherrer2014-localPolicySearch,Neu17-unified}.
For the last two decades, CPI has inspired many studies on \emph{ensuring monotonic policy improvement}.
However, those studies (including CPI itself) are mostly theory-oriented and hardly applicable to practical scenarios, in that maximizing the lowerbound requires solving several state-action-space-wise maximization problems, e.g. estimating the maximum distance between two arbitrary policies.
One significant factor causing the complexity might be its excessive generality~\cite{Kakade02,pirotta13}; 
these bounds do not focus on any particular class of value-based RL algorithms, and hence without further assumptions the problem cannot be simplified.

Another recent trend of developing algorithms robust to the oscillation is by \emph{introducing regularizers into the reward function}.
For example, by maximizing reward as well as Shannon entropy of policy \cite{ziebart2010-phd}, the optimal policy becomes a multi-modal Boltzmann softmax distribution which avoids putting unit probability mass on the greedy but potentially sub-optimal actions corrupted by noise or error, significantly enhancing the robustness since optimal actions always have nonzero probabilities of being chosen. 
On the other hand, the introduction of Kullback-Leibler (KL) divergence \cite{todorov2006linearly} has recently been identified to yield policies that average over all past value functions and errors, which enjoys state-of-the-art error dependency theoretically \cite{vieillard2020leverage}. 
Though entropy-regularized algorithms have superior finite-time bounds and enjoy strong empirical performance, they do not guarantee to reduce policy oscillation since degradation during learning can still persist \cite{Nachum2017-TrustPCL}.

It is hence natural to raise the question of whether the practically intractable lowerbounds from the monotonic improvement literature can benefit from entropy regularization if we restrict ourselves to the entropy-regularized policy class. 
By noticing that the policy interpolation and entropy regularization actually perform regularization in different aspects, i.e. in the stochastic policy space and reward function, we answer this question by affirmative. 
We show focusing on the class of entropy-regularizede policies significantly simplifies the problem as a very recent result indicates a sequence of entropy-regularized policies has bounded KL divergence \cite{kozunoCVI}. 
This result sheds light on approximating the intractable lowerbounds from the monotonic improvement algorithms since many quantities are related to the maximum distance between two arbitrary policies.

% In this paper, in order to develop more tractable bounds, we focus on an RL class known as entropy-regularized value-based methods~\cite{azar2012dynamic,Fox2016,haarnoja-SAC2017a,haarnoja-SAC2018}, where the entropy of policies is introduced in the reward function for regularizing policy updates. Sample efficiency and error tolerance have been well studied~\cite{kozunoCVI}; however, their monotonic improvement has not been explored.

In this paper, we aim to tackle the policy oscillation problem by ensuring monotonic improvement via optimizing a more tractable lowerbound.
This novel entropy regularization aware lower bound of policy improvement depends only the expected policy advantage function.
We call the resultant algorithm \emph{cautious policy programming (CPP)}.
CPP leverages this lower bo\-und as a criterion for adjusting the degree of a policy update for alleviating policy oscillation. 
By introducing heuristic designs suitable for nonlinear approximators, CPP can be extended to working with deep networks. 
The extensions are compared with the state-of-the-art algorithm \cite{Vieillard-2020DCPI} on monotonic policy improvement. 
We demonstrate that our approach can trade off performance and stability in both didactic classic control problems and challenging Atari games.

The contribution of this paper can be succinctly summarized as follows:
\begin{itemize}
  \item we develop an easy-to-use lowerbound for ensuring monotonic policy improvement in RL.
  \item we propose a novel scalable algorithm CPP which optimizes the lowerbound. 
  \item CPP is validated to reduce policy oscillation on high-dimensional problems which are intractable for prior methods.
\end{itemize}
Here, the first and second points are presented in Sec. \ref{sec:proposedMethod}, after a brief review on related work in Sec. \ref{sec:relatedWork} and preliminary in Sec. \ref{sec:preliminary}. 
The third point is inspected in  Sec. \ref{sec:experimental} which presents the results.
CPP has touched upon many related problems, and we provide in-depth discussion in Sec. \ref{sec:discussion}. 
The paper is concluded in Sec. \ref{sec:conclusion}.
To not interrupt the flow of the paper, we defer all proofs until the Appendix.

\section{Related Work}\label{sec:relatedWork}

The policy oscillation phenomenon, also termed \emph{overshooting} by~\cite{wagner2011} and referred to as degraded performance of updated policies, frequently arises in approximate policy iteration algorithms~\cite{Bertsekas2011} and can occur even under asymptotically converged value functions~\cite{wagner2011}. It has been shown that aggressive updates with sampling and update errors, together with restricted policy spaces, are the main reasons for policy oscillation~\cite{pirotta13}. %In real-world RL applications, policy oscillation might pose a threat to the system or its environment due to inferior performance.
In modern applications of RL, policy oscillation becomes an important issue when learning with deep networks when various sources of errors have to been taken in to account. It has been investigated by \cite{Fujimoto18-addressingApproximationError,fu2019-diagnosis} that those errors are the main cause for typical oscillating performance with deep RL implementations.

To attenuate policy oscillation, the seminal algorithm conservative policy iteration (CPI) \cite{Kakade02} propose to perform regularization in the stochastic policy space, whereby the greedily updated policy is interpolated with the current policy to achieve less aggressive updates. %However, the sample complexity of CPI is high that for one update, CPI would require as many samples as for other algorithms would obtain the optimal policy~\cite{thomas2015safe,akrour-monotonic-2016}.
{\color{black}
CPI has inspired numerous conservative algorithms that enjoy strong theoretical guarantees \cite{pirotta13,Pirotta13-adaptiveStepSizePG,abbasi-improvement16,Metelli18-configurable} to improve upon CPI by proposing new lower bounds for policy improvement. However, since their focus is on general Markov decision processes (MDPs), deriving practical algorithms based on the lower bounds is nontrivial and the proposed lower bounds are mostly of theoretical value. 
Indeed, as admitted by the authors of \cite{papini20-balanceSpeedStabilityPG} that a large gap between theory and practice exists, as manifested by the their experimental results that even for a simple Cartpole environment, state-of-the-art algorithm failed to deliver attenuated oscillation and convergence speed comparable with heuristic optimization scheme such as Adam \cite{Adam}.
This might explain why \emph{adaptive coefficients} must be introduced in~\cite{Vieillard-2020DCPI} to extend CPI to be compatible with deep neural networks. To remove this limitation, our focus on entropy-regularized MDPs allows for a straightforward algorithm based on a novel, significantly simplified lower bound. %However, SPI focused on general stationary policies whose difference is difficult to bound, if not impossible. This policy difference term appears in several optimizations over the entire state space, which is the reason SPI is intractable for problems with continuous or/and high dimensional state spaces.
}

Another line of research toward alleviating policy oscillation is to incorporate regularization as a penalty into the reward function, leading to the recently booming literature on entropy-regularized MDPs \cite{azar2012dynamic,Fox2016,kozunoCVI,haarnoja-SAC2017a,vieillard2020leverage,Mei2019-principledEntropySearch}. 
Instead of interpolating greedy policies, the reward is augmented with entropy of the policy, such as Shannon entropy for more diverse behavior and smooth optimization landscape \cite{ahmed19-entropy-policyOptimization},
 or Kullback-Leibler (KL) divergence for enforcing policy similarity between policy updates and hence achieving superior sample efficiency \cite{Uchibe2018,uchibe2021-forwardBackward}. 
The Shannon entropy renders the optimal policy of the regularized MDP stochastic and multi-modal and hence robust against errors and noises in contrast to the deterministic policy that puts all probability mass on a single action \cite{haarnoja-SAC2018}.
On the other hand, augmenting with KL divergence shapes the optimal policy an average of all past value functions, which is significantly more robust than a single point estimate.
Compared to the CPI-based algorithms, entropy-regularized algorithms do not have guarantee on per-update improvement. But they have demonstrated state-of-the-art empirical successes on a wide range of challenging tasks \cite{cui2017kernel,TSURUMINE2019,ZHU2020CEP,ZHU2022CCE}.
To the best of the authors' knowledge, unifying those two regularization schemes has not been considered in published literature before.

{\color{black}
It is worth noting that, inspired by \cite{Kakade02}, the concept of monotonic improvement has been exploited also in policy search scenarios~\cite{trpo-schulman15,akrour-monotonic-2016,Lior19-adaptiveTRPO,mei20b-globalConvergence,papini20-balanceSpeedStabilityPG}. 
However, there is a large gap between theory and practice in those policy gradient methods. 
On one hand, though \cite{trpo-schulman15,schulman2017proximal} demonstrated good empirical performance, their relaxed trust region is often too optimistic and easily corrupted by noises and errors that arise frequently in the deep RL setting: as pointed out by \cite{engstrom2020implementation}, the trust region technique itself alone fails to explain the efficiency of the algorithms and lots  of code-level  tricks are necessary.
On the other hand,  exactly following the guidance of monotonic improving gradient does not lead to tempered oscillation and better performance even for simple problems \cite{Papini2017-adaptiveBatchSizePG,papini20-balanceSpeedStabilityPG}.
Another  shortcoming of policy gradient methods is they focus on local optimal policy with strong dependency on initial parameters. 
On the other hand, we focus on value-based RL that searches for global optimal policies.
}

\section{Preliminary}\label{sec:preliminary}

%{\color{black}{[Put overview of this section]}}
%In this section we introduce the basics of RL and several important lemmas for later derivation of our proposed method.

\subsection{Reinforcement Learning}

RL problems can be formulated by MDPs expressed by the quintuple $(\mathcal{S},\mathcal{A},\mathcal{T},\mathcal{R},\gamma)$, where $\mathcal{S}$ denotes the state space, $\mathcal{A}$ denotes the finite action space, and $\mathcal{T}$ denotes transition dynamics such that $\mathcal{T}_{ss'}^{a}:=\mathcal{T}(s'|s,a)$ represents the transition from state $s$ to $s'$ with action $a$ taken. $\mathcal{R} = r^{\,\,a}_{ss'}$ is the immediate reward associated with that transition. In this paper, we consider $r^{\,\,a}_{ss'}$ as bounded in the interval $[-1, 1]$. $\gamma \in (0,1)$ is the discount factor. For simplicity, we consider the infinite horizon discounted setting with a fixed starting state $s_{0}$. A policy $\pi$ is a probability distribution over actions given some state. We also define the stationary state distribution induced by $\pi$ as $d^{\pi}(s) = (1-\gamma)\sum_{t=0}^{\infty}\gamma^{t}\mathcal{T}({s_{t} = s|s_{0}, \pi})$. 

RL algorithms search for an optimal policy $\pi^{*}$ that maximizes the state value function for all states $s$:
\begin{align*}
  \pi^{*} := \argmax_{\pi} V^{\pi}(s) = \argmax_{\pi} \mathbb{E}\left[\sum_{t=0}^{\infty}\gamma^{t} r_{t} \big| s_{0} = s \right],
\end{align*}
where the expectation is with respect to the transition dynamics $\mathcal{T}$ and policy $\pi$. The state-action value function $Q^{\pi}$ is more frequently used in the control context:
\begin{align*}
  Q^{\pi^{*}}(s, a) = \max_{\pi} \mathbb{E}\left[\sum_{t=0}^{\infty}\gamma^{t} r_{t} \big| s_{0} = s, a_{0} = a \right].
\end{align*}

\iffalse
RL methods search for an optimal stationary policy $\pi^{*}$ such that the expected long-term discounted reward is maximized, over all states:
\begin{align}
  \begin{split}
    V_{\pi^{*}}(s) = \max_{\pi}\mathbb{E}_{\mathcal{T}}\big[ \sum_{t=0}^{\infty}\gamma^{t} (r_{ss'}^{a})_{t} \big| s_{0} = s \big].
  \end{split}
\end{align}
It is known that $V_{\pi^{*}}$ solves the following system of equations known as the Bellman optimality:
\begin{align}
\begin{split}
V_{{\pi}^{*}}(s)=\max_{\pi}{\sum_{\substack{a \in \mathcal{A} \\ s' \in \mathcal{S}}}\pi(a|s)\bigg[\mathcal{T}_{ss'}^{a}\big(r_{ss'}^{a}+\gamma V_{{\pi}^{*}}(s')\big)\bigg]}.
\label{sys_Vbellman}
\end{split}
\end{align}
The state-action value function $Q_{\pi^{*}}(s,a)$ is more frequently used in control context:
\begin{align}
\begin{split}
Q_{\pi^{*}}(s,a)&=\max_{\pi}{\sum_{s'\in\mathcal{S}}{\mathcal{T}_{ss'}^{a}\big(r_{ss'}^{a}+\gamma\sum_{a'\in\mathcal{A}}{\pi(a'|s')Q_{\pi^{*}}(s',a')}\big)}}.
\label{sys_Qbellman}
\end{split}
\end{align}
\fi

\subsection{Lower Bounds on Policy Improvement}

%The main contribution of this paper is a novel RL algorithm that ensures per-iteration policy improvement $||J^{K+1}-J^{K}|| \geq 0$. We defer the approach until next section and provide several useful lemmas here, 
%In this paper, we want to lower bound the policy improvement after every policy update. The criterion for improvement is defined by \emph{policy return}.
To frame the monotonic improvement problem, we introduce the following lemma that formally defines the criterion of policy improvement of some policy $\pi'$ over $\pi$:

\begin{lemma}{\cite{Kakade02}}\label{thm:kakade}
  {For any stationary policies $\pi'$ and $\pi$, the following equation holds}:
  \begin{align}
    \begin{split}
      &\Delta J^{\pi'}_{\pi, d^{\pi'}} := J^{{\pi'}}_{d} - J^{{\pi}}_{d} =  {\sum_{s}{d^{\pi'}(s)\sum_{a}{\pi'(a|s){A_{\pi}(s,a)}}}}, \\
      %&= \sum_{s}d^{\pi'}(s) A^{\pi'}_{\pi}(s) =: A^{\pi'}_{\pi, d}\\
      &\text{where } J^{{\pi'}}_{d} := \mathbb{E}_{s_{0},a_{0},\dots}{\bigg[ (1 - \gamma) \sum_{t=0}^{\infty}\gamma^{t}r_{t}\bigg]} = \sum_{s}{d^{{\pi'}}(s)}\sum_{a}{{\pi'}(a|s)r^{a}_{ss'}},
    \end{split}
    \label{kakade_idty}
  \end{align}
\end{lemma}
$J$ is the discounted cumulative reward, and $A_{\pi}(s,a) := \quad$ $ Q_{\pi}(s,a) - V_{\pi}(s)$ is the advantage function. Though Lemma \ref{thm:kakade} relates policy improvement to the expected advantage function, pursuing policy improvement by directly exploiting Lemma \ref{thm:kakade} is intractable as it requires comparing $\pi'$ and $\pi$ point-wise for infinitely many new policies. Many existing works~\cite{Kakade02,pirotta13,trpo-schulman15} instead focus on finding a $\pi'$ such that the right-hand side of Eq. (\ref{kakade_idty}) is lower bounded. To alleviate policy oscillation brought by the greedily updated policy $\tilde{\pi}$,~\cite{Kakade02} proposes adopting \emph{partial update}:
\begin{align}
  \begin{split}
\pi' = \zeta\tilde{\pi} + (1-\zeta)\pi.
  \label{mixture_policy}
  \end{split}
\end{align}
Eq. (\ref{mixture_policy}) corresponds to performing regularization in the stochastic policy space by interpolating the greedy policy and the current policy to achieve conservative updates. 

{\color{black}
The concept of linearly interpolating policies has inspired many algorithms that enjoy strong theoretical guarantees \cite{pirotta13,Metelli18-configurable,akrour-monotonic-2016}.
However, those algorithms are mostly of theoretical value and have only been applied to small problems due to intractable optimization or estimation when the state-action space is high-dimensional/continuous.
Indeed, as admitted by the authors of \cite{papini20-balanceSpeedStabilityPG}, there is a large gap between theory and practice when using algorithms based on policy regularization Eq. (\ref{mixture_policy}): even on a simple CartPole problem, a state-of-the-art algorithm fail to compete with heuristic optimization technique. 
Like our proposal in this paper, a very recent work \cite{Vieillard-2020DCPI} attempts to bridge this gap by proposing heuristic coefficient design for learning with deep networks. 
We discuss the relationship between it and the CPP in Section \ref{sec:approximate_zeta}.

In the next section, we detail the derivation of the proposed lower bound by exploiting entropy regularization. 
This novel lower bound allows us to significantly simplify the intractable optimization and estimation in prior work and provide a scalable implementation.
}

\section{Proposed Method}\label{sec:proposedMethod}

%{\color{black}{[Put overview of our approach]}}
%In this section we detail our proposed method. First a general formulation of entropy-regularized RL is introduced, followed by a lemma that bounds the maximum distance between policies of entropy-regularized update. Finally we propose the main theorem and a novel algorithm for ensuring monotonic improvement.

This section features the proposed novel lower bound on which we base a novel algorithm for ensuring monotonic policy improvement.

\subsection{Entropy-regularized RL}

In the following discussion, we provide a general formulation for entropy-regularized algorithms~\cite{azar2012dynamic,haarnoja-SAC2018,kozunoCVI}. 
At iteration \emph{K}, the entropy of current policy $\pi_{K}$ and the Kullback-Leibler (KL) divergence between $\pi_{K}$ and some baseline policy $\bar{\pi}$ are added to the value function:
\begin{align}
\begin{split}
V_{\bar{\pi}}^{\pi_{K}}(s) &:= \sum_{\substack{a \in \mathcal{A} \\ s' \in \mathcal{S}}}\pi(a|s)\bigg[\mathcal{T}_{ss'}^{a}\big(r_{ss'}^{a}+\gamma V^{\pi_{K}}_{\bar{\pi}}(s')\big) - \mathcal{I}_{\bar{\pi}}^{\pi_{K}}\bigg], \\
\mathcal{I}_{\bar{\pi}}^{\pi_{K}} &= -\tau\log{\pi_{K}(a|s)} - \sigma\log{\frac{\pi_{K}(a|s)}{\bar{\pi}(a|s)}},
\label{sys_DPPbellman}
\end{split}
\end{align}
where $\tau$ controls the weight of the entropy bonus and $\sigma$ weights the effect of KL regularization. The baseline policy $\bar{\pi}$ is often taken as the policy from previous iteration $\pi_{K-1}$. 
{\color{black}
Based on \cite{Nachum2017-bridgeGap,Nachum2017-TrustPCL}, we know the state value function $V^{\pi_{K}}_{\bar{\pi}}$ defined in Eq. (\ref{sys_DPPbellman}) and state-action value function $Q^{\pi_{K}}_{\bar{\pi}}$ also satisfy the Bellman recursion:
\begin{align*}
    Q^{\pi_{K}}_{\bar{\pi}}(s,a) := r_{ss'}^{a} + \gamma\sum_{s'}\mathcal{T}_{ss'}^{a}V^{\pi_{K}}_{\bar{\pi}}(s').
\end{align*}
For notational convenience, in the remainder of this paper, we use the following definition:
\begin{align}
  \alpha := \frac{\tau}{\tau+\sigma}, \quad \beta:=\frac{1}{\tau+\sigma}.
  \label{eq:coef_def}
\end{align}
 An intuitive explanation to Eq. (\ref{sys_DPPbellman}) is that the entropy term endows the optimal policy with multi-modal policy behavior \cite{haarnoja-SAC2017a} by placing nonzero probability mass on every action candidate, hence is robust against error and noise in function approximation that can easily corrupt the conventional deterministic optimal policy \cite{Puterman1994}.
 On the other hand, KL divergence provides smooth policy updates by limiting the size of the update step \cite{azar2012dynamic,kozunoCVI,trpo-schulman15}. 
 Indeed, it has been recently shown that augmenting the reward with KL divergence renders the optimal policy an exponential smoothing of all past value functions \cite{vieillard2020leverage}.
 Limiting the update step plays a crucial role in the recent successful algorithms since it prevents the aggressive updates that could easily be corrupted by errors \cite{Fujimoto18-addressingApproximationError,fu2019-diagnosis}.
 It is worth noting that when the optimal policy is attained, the KL regularization term becomes zero. 
 Hence in Eq. (\ref{sys_DPPbellman}), the optimal policy maximizes the cumulative reward while keeping the entropy high.

}

\subsection{Entropy-regularization-aware Lower Bound}\label{nlb}

{\color{black}

Recall in Eq. (\ref{mixture_policy}) performing regularization in the stochastic policy space for the greedily updated policy $\tilde{\pi}$ requires preparing an reference policy $\pi$.
This policy could be from expert knowledge or previous policies.
The resultant $\pi'$,  has guaranteed monotonic improvement which we formulate as the following lemma:

\begin{lemma}[\cite{pirotta13}]\label{thm:SPI}
  {Provided that policy $\pi'$ is generated by partial update Eq. (\ref{mixture_policy}), $\zeta$ is chosen properly, and $A_{\pi, d^{{\pi}}}^{\tilde{\pi}} \geq 0$, then the following policy improvement is guaranteed:} 
\begin{align}
  \begin{split}
&\Delta J^{\pi'}_{\pi, d^{\pi'}} \geq \frac{\big((1-\gamma)A_{\pi,d^{\pi}}^{\tilde{\pi}}\big)^{2}}{2\gamma\delta\Delta A^{\tilde{\pi}}_{\pi}}, \\
\text{with } & \zeta = \min{(1, \zeta^{*})},\\
\text{where } &\zeta^{*}=\frac{(1-\gamma)^{2}A^{\tilde{\pi}}_{{\pi, d^{\pi}}}}{\gamma\delta\Delta A^{\tilde{\pi}}_{\pi}},\\
&\delta=\max_{s}{\left|\sum_{a\in\mathcal{A}}\big(\tilde{\pi}(a|s)-\pi(a|s)\big)\right|},\\
&\Delta A^{\tilde{\pi}}_{\pi}=\max_{s, s'}{|A^{\tilde{\pi}}_{\pi}(s)-A^{\tilde{\pi}}_{\pi}(s')}|,
  \end{split}
  \label{J_first_exact}
\end{align}
where $A_{\pi,d^{\pi}}^{\tilde{\pi}} := \sum_{s}d^{\pi}(s)\sum_{a} \adaBracket{\tilde{\pi}(a|s) - \pi(a|s)}Q_{\pi}(s,a)$.
\end{lemma}
\begin{proof}
  See Section \ref{apdx:lemma2} for the proof.
\end{proof}
The interpolated policy $\pi'$ optimizes the bound and the policy improvement is a negative quadratic function in $\zeta$.
%The interpolated policy $\pi'$ is optimal in the sense that it is the maximizer of policy improvement function negative quadratic in $\zeta$, spanned by $\tilde{\pi}$ and $\pi$.
However, this optimization problem is highly non-trivial as $\delta$ and $\Delta A^{\tilde{\pi}}_{\pi}$ require searching the entire state-action space. 
This challenge explains why CPI-inspired methods have only been applied to small problems with low-dimensional state-action spaces \cite{pirotta13,Metelli18-configurable,papini20-balanceSpeedStabilityPG}.

When the expert knowledge is not available, we can simply choose previous policies.
Specifically, at any iteration $K$, we want to ensure monotonic policy improvement given policy $\pi_{K}$.
We propose constructing a new monotonically improving policy as:
\begin{align}
  \begin{split}
    \tilde{\pi}_{K+1} = \zeta\pi_{K+1} + (1-\zeta)\pi_{K}.
  \end{split}
  \label{mixture_cvi}
\end{align}
It is now clear by comparing Eq. (\ref{mixture_policy}) with Eq. (\ref{mixture_cvi}) that our proposal takes $\pi', \tilde{\pi}, \pi$ as $\tilde{\pi}_{K+1}, \pi_{K+1}, \pi_{K}$, respectively. It is worth noting that ${\pi}_{K+1}$ is the \emph{updated policy that has not been deployed.} %The interpolation renders $\tilde{\pi}_{K+1}$ more conservative as it is closer to the baseline $\pi_{0}$ than $\pi_{K+1}$.

However, the intractable quantities $\delta$ and $\Delta A^{\tilde{\pi}}_{\pi}$ in Lemma \ref{thm:SPI} are still an obstacle to deriving a scalable algorithm.
Specifically, by writing the component $A^{\tilde{\pi}}_{\pi}(s)$ of $\Delta A^{\tilde{\pi}}_{\pi}$ as
\begin{align*}
  A^{\tilde{\pi}}_{\pi}(s) = \sum_{a}  \big( \tilde{\pi}(a|s) - \pi(a|s) \big)Q_{\pi}(s,a),
\end{align*}
we see that both $\delta$ and $\Delta A^{\tilde{\pi}}_{\pi}$ require accurately estimating the total variation between two policies. 
This could be difficult without enforcing constraints such as gradual change of policies.
Fortunately, by noticing that the consecutive entropy-regularized policies $\pi_{K+1}, \pi_{K}$  have \emph{bounded} total variation,  we can leverage the boundedness to bypass the intracatable estimation.

\begin{lemma}[\cite{kozunoCVI}]\label{thm:KL}
  {For any policies $\pi_{K}$ and $\pi_{K+1}$ generated by taking the maximizer of Eq. (\ref{sys_DPPbellman}), the following bound holds for their maximum total variation}:
\begin{align}
  \begin{split}
    &\max_{s}{D_{TV}\left(\pi_{K+1}(\cdot|s) \,||\, \pi_{K}(\cdot|s) \right) } \leq \\
    &\qquad \qquad \qquad \qquad \min \left\{  \sqrt{1 - e^{- 4 B_{K} - 2 C_{K}}}, \sqrt{8 B_{K} + 4C_{K}} \right\} , \\
    & \qquad \text{where } B_{K}=\frac{1-\gamma^{K}}{1-\gamma}\epsilon\beta , \,\,\,\, C_{K} = \beta r_{max} \sum_{k=0}^{K-1}{\alpha^{k}\gamma^{K-k-1}},
  \end{split}
  \label{CVI_kl}
\end{align}
\emph{$K$ denotes the current iteration index and $0\leq k\leq K-1$ is the loop index. 
$\epsilon$ is the uniform upper bound of error.} 
\end{lemma} 
\begin{proof}
  See Section \ref{apdx:lemma3} for the proof.
\end{proof}

Lemma \ref{thm:KL} states that, entropy-regularized policies have boun\-ded total variation (and hence bounded KL divergence by Pinsker's and Kozuno's inequality \cite{kozunoCVI}).
This bound allows us to bypass the intractable estimation in Lemma \ref{thm:SPI} and approximate $\tilde{\pi}_{K+1}$ that optimizes the lowerbound.
We formally state this result in the Theorem \ref{thm:main} below.

For convenience, we assume there is no error, i.e. $B_{K} = 0$. 
Setting $B_{K} = 0$ is only for the ease of notation of our latter derivation. Our results still hold by simply replacing all appearance of $C_{K}$ to $B_{K} + C_{K}$. On the other hand, in implementation it requires a sensible choice of upper bound of error which is typically difficult especially for high dimensional problems and with nonlinear function approximators.
    Fortunately, by the virtue of KL regularization in Eq. (\ref{sys_DPPbellman}), it has been shown in \cite{azar2012dynamic,Vieillard2020Momentum} that if the sequence of errors is a martingale difference under the natural filtration, then the summation of errors asymptotically cancels out. Hence it might be safe to simply set $B_{K} = 0$ if we assume the martingale difference condition.
%Intuitively, the agent collects samples and updates the policy to $\pi_{K+1}$. However, instead of directly deploying this policy, we interpolate it with $\pi_{K}$ by $\zeta$ to obtain $\tilde{\pi}_{K+1}$. 
% Among the linear class of policies spanned by $\pi_{K+1}$ and $\pi_{K}$, we select the maximizer $\tilde{\pi}_{K+1}$ that provides the largest improvment, as formally stated by Theorem \ref{thm:main} below:

\begin{theorem}\label{thm:main}
  {Provided that partial update Eq. (\ref{mixture_cvi}) is adopted, $A^{{\pi_{K+1}}}_{\pi_{K},d^{\pi_{K}}} \geq 0$, and $\zeta$ is chosen properly as specified below, then any maximizer policy of Eq. (\ref{sys_DPPbellman}) guarantees the following improvement that depends only on $\alpha, \beta, \gamma$ and $A^{{\pi_{K+1}}}_{\pi_{K},d^{\pi_{K}}}$ after any policy update:}

\begin{align}
  \begin{split}
\Delta J^{\tilde{\pi}_{K+1}}_{\pi_{K},d^{\tilde{\pi}_{K+1}}}  &\geq \frac{\big(1-\gamma)^{3}(A_{\pi_{K},d^{\pi_{K}}}^{{\pi_{K+1}}})^{2}}{4 \gamma} \max \left\{\frac{1}{1-e^{-2 C_{K}}} \,\, , \,\, \frac{1}{4 C_{K} } \right\},\\
  \text{with } \zeta &= \min{(1, \zeta^{*})}, \quad C_{K} = \beta\sum_{k=0}^{K-1}{\alpha^{k}\gamma^{K-k-1}}, \\
\text{where } \zeta^{*} &= \frac{(1-\gamma)^{3}A^{{\pi_{K+1}}}_{{\pi_{K},d^{\pi_{K}}}}}{2 \gamma } \max\left\{  \frac{1}{1-e^{-2 C_{K}}}, \frac{1}{4 C_{K}} \right\},
\label{J_first_improved}
  \end{split}
\end{align}
$\alpha, \beta$ are defined in Eq. (\ref{eq:coef_def}) and
\begin{align}
    & A^{\pi_{K+1}}_{\pi_{K}, d^{{\pi}_{K}}} := \sum_{s}d^{{{\pi}_{K}}} (s) \, A^{{\pi_{K+1}}}_{\pi_{K}} (s), \label{eq:statAdv}\\
    & A^{{\pi_{K+1}}}_{\pi_{K}} (s) := \sum_{a}  \big( {\pi_{K+1}}(a|s) - \pi_{K}(a|s) \big)Q_{\pi_{K}}(s,a) \label{eq:expAdv}
\end{align}
are the expected policy advantage, and the policy advantage function, respectively.
\end{theorem} 

\begin{proof}
  See Section \ref{apdx:thm4} for the proof.
\end{proof}

}

%Unlike the difficult $\epsilon$-accurate estimates in practice~\cite{Kakade02,pirotta13}, estimating $A^{\tilde{\pi}}_{{\pi},d}$ is straightforward.

While theoretically we need to compare $1 - e^{-2 C_{K}}$ and $4 C_{K}$ when computing $\zeta^{*}$, in implementation the exponential function $e^{-2 C_{K}}$ might be sometimes close to 1 and hence causing numerical instability. 
Hence in the rest of the paper we shall stick to using the constant $C_{K}$ rather than the exponential function.

In the lower bound Eq. (\ref{J_first_improved}), only $A^{{\pi_{K+1}}}_{{\pi_{K}},d^{{\pi}_{K}}}$ needs to be estimated. It is worth noting that $\forall s, A^{{\pi_{K+1}}}_{{\pi_{K}}} (s) \geq 0$ is a straightforward criterion that is naturally satisfied by the greedy policy improvement of policy iteration when computation is exact. 
To handle the negative case caused by error or approximate computations, we can simply stack more samples to reduce the variance, as will be detailed in Sec. \ref{sec:CPP_FA}. %It is worth noting that for circumventing the maximization problems $\delta$ and $\Delta A^{\tilde{\pi}}_{\pi}$ in Eq. (\ref{J_first_exact}), the price we pay is the dependence on $C_{K}$.

\subsection{The CPP Policy Iteratiion}

We now detail the structure of our proposed algorithm based on Theorem \ref{thm:main}. 
Specifically, value update, policy update, and stationary distribution estimation are introduced, followed by discussion on a subtlety in practice and two possible solutions.

{\color{black}

Following \cite{scherrer15-AMPI}, CPP can be written in the following succinct policy iteration style:
\begin{align}
  \begin{split}
    \text{CPP} = \begin{cases}
      \pi_{K+1} \leftarrow \mathcal{G}{Q^{\pi_{K}}_{\bar{\pi}}} & \\
      Q_{\pi_{K+1}} \leftarrow (T_{\pi_{K+1}})^{m} Q_{\pi_{K}} & \\
      \zeta = \min \left\{(4 C_{K})^{-1}{C_{\gamma} A^{{\pi_{K+1}}}_{{\pi_{K},d^{\pi_{K}}}}} \,\, , \,\, 1\right\} \\
      \tilde{\pi}_{K+1} \leftarrow \zeta \pi_{K+1} + (1 - \zeta) \pi_{K},
    \end{cases}
  \end{split}
  \label{eq:cpp_pi}
\end{align}
where $C_{\gamma} := \frac{(1-\gamma)^3}{2\gamma}$ is the horizon constant. Note that for numerical stability we stick to using $(4C_K)^{-1}$ as the entropy-bounding constant rather than using $\frac{1}{1 - e^{-2C_K}}$.

Like CPI, CPP can obtain \emph{global optimal policy} rather than just achieving monotonic improvement (which might still converge to a local optimum) by the argument of \cite{Scherrer2014-localPolicySearch}.
The first step corresponds to the greedy step of policy iteration, the second step policy estimation step, third step computing interpolation coefficient $\zeta$ and the last step interpolating the policy.

\subsubsection{Policy Improvement and Policy Evaluation}\label{sec:policyIter}

The first two steps are standard update and estimation steps of policy iteration algorithms \cite{Sutton-RL2018}. 
The subscript of $Q^{\pi_{K}}_{\bar{\pi}}$ indicates it is entropy-regularized as introduced in Eq. (\ref{sys_DPPbellman}).

The policy improvement step consists of evaluating $\mathcal{G}{Q^{\pi_{K}}_{\bar{\pi}}}$, which is the greedy operator acting on $Q^{\pi_{K}}_{\bar{\pi}}$. 
By the Fenchel conjugacy of Shannon entropy and KL divergence, $\mathcal{G}{Q^{\pi_{K}}_{\bar{\pi}}}$ has a closed-form solution \cite{kozunoCVI,Beck2017-firstOrder}:
\begin{align*}
  \mathcal{G}{Q^{\pi_{K}}_{\bar{\pi}}} (a|s) = \frac{\bar{\pi}(a|s)^{\alpha}\exp\adaBracket{\beta Q^{\pi_{K}}_{\bar{\pi}} (s,a) }}{\sum_{b} \bar{\pi}(b|s)^{\alpha}\exp\adaBracket{\beta Q^{\pi_{K}}_{\bar{\pi} } (s,b) }},
\end{align*}
where $\alpha, \beta$ were defined in Eq. (\ref{eq:coef_def}).

The policy evaluation step estimates the value of current policy $\pi_{K+1}$ by repeatedly applying the Bellman operator $T_{\pi_{K+1}}$:
\begin{align}
  \begin{split}
  &\qquad  \qquad (T_{\pi_{K+1}})^{m} Q_{\pi_{K}} := \underbrace{T_{\pi_{K+1}}\dots T_{\pi_{K+1}}}_{m  \text{ times}} Q_{\pi_{K}}, \\
  & T_{\pi_{K+1}} Q_{\pi_{K}} := r_{ss'}^{a} + \gamma\sum_{s'}\mathcal{T}_{ss'}^{a} \sum_{a}\pi_{K+1}(a|s') Q^{\pi_{K}}_{\bar{\pi}}(s',a).
\end{split}
\label{eq:policy_evaluation}
\end{align}
Note that $m = 1, \infty$ correspond to the value iteration and policy iteration, respectively \cite{Bertsekas:1996:NP:560669}. 
Other interger-valued $m \in [2, \infty)$ correspond to the approximate modified policy iteration \cite{scherrer15-AMPI}.

Now in order to estimate $A_{\pi_{K},d^{\pi_{K}}}^{{\pi_{K+1}}}$ in Theorem \ref{thm:main}, both $A_{\pi_{K}}^{{\pi_{K+1}}}$ and $d^{{\pi}_{K}}$ need to be estimated from samples. 
Estimating $A_{\pi_{K}}^{{\pi_{K+1}}}(s)$ is straightforward by its definition in Eq. (\ref{eq:expAdv}).
We can first compute $Q_{\pi_{K}}(s,a) - V_{\pi_{K}}(s), \,\forall s,a$ for the current policy, and then update the policy to obtain $\pi_{K+1}(a|s)$.
On the other hand,  sampling with respect to $d^{{\pi}_{K}}$ results in an on-policy algorithm, which is expensive.
We provide both on- and off-policy implementations of CPP in the following sections, but in principle off-policy learning algorithms can be applied to estimate $d^{{\pi}_{K}}$ by exploiting techniques such as importance sampling (IS) ratio \cite{precup2000eligibility}.
%$A_{\pi_{K}}^{{\pi_{K+1}}}(s) = \sum_{a}\pi_{K+1}(a|s)\big(Q_{\pi_{K}}(s,a) - V_{\pi_{K}}(s)\big)$.

\subsubsection{Leveraging Policy Interpolation}\label{sec:leverage_interpolation}

Computing $\zeta$ in Eq. (\ref{J_first_improved}) involves the horizon constant $C_{\gamma} := \frac{(1-\gamma)^3}{2\gamma}$ and policy difference bound constant $C_{K}$.
The horizon constant is effective in DP scenarios where the total number of timesteps is typically small, but might not be suitable for learning with deep networks that feature large number of timesteps: a vanishingly small $C_{\gamma}$ will significantly hinder learning, hence it should be removed in deep RL implementations. We detail this consideration in Section \ref{sec:CPP_FA}.

The updated policy $\pi_{K+1}$ in Eq. (\ref{eq:cpp_pi}) cannot be directly deployed since it has not been verified to improve upon $\pi_{K}$. 
We interpolate between $\pi_{K+1}$ and $\pi_{K}$ with coefficient $\zeta$ such that the resultant policy $\tilde{\pi}_{K+1}$ by finding the maximizer of a negative quadratic function in $\zeta$.
The maximizer $\zeta^*$ optimizes the lowerbound $\Delta J^{\tilde{\pi}_{K+1}}_{\pi_{K},d^{\tilde{\pi}_{K+1}}}$.
Here, $\zeta$ is optimally tuned and dynamically changing in every update. It reflects the \emph{cautiousness} against policy oscillation, i.e., how much we trust the updated policy $\pi_{K+1}$. Generally, at the early stage of learning, $\zeta$ should be small in order to explore conservatively.

However, a major concern is that Lemma \ref{thm:KL} holds only for Boltzmann policies, while the interpolated policies are generally no longer Boltzmann. 
In practice, we have two options for handling this problem: 
\begin{enumerate}
  \item we use the interpolated policy only for collecting samples  (i.e. behavior policy) but not for computing next policy;
  \item we perform an additional projection step to project the interpolated policy back to the Boltzmann class as the next policy.
\end{enumerate}
{\color{black} 
The first solution might be suitable for relatively simple problems where the safe exploration is required:
the behavior policy is conservative in exploring when $\zeta \!\approx\! 0$.
But learning can still proceed even with such small $\zeta$.
Hence this scheme suits problems where interaction with the environment is crucial but progress is desired.}
On  the  other  hand,  the  second scheme is more natural since the off-policyness  caused  by  the mismatch between the behavior  and  learning  policy  might be compounded by high dimensionality.  The increased mismatch might be perturbing  to  performance.   
In  the  following  section,  we  introduce  CPP using  linear  function  approximation  for  the  first  scheme  and  deep  CPP for the second scheme.

% We now introduce the theoretical foundation of the second scheme, which states that there must exist a Boltzmann policy that is \emph{equivalent} to the linearly interpolated policy.
% \begin{theorem}\emph{(\cite[Theorem 2.8]{ziebart2010-phd})}\label{theorem:mixture}
%   Let $\pi^{(1)},$  $\pi^{(2)},$ $\dots, \pi^{(n)}$  be an arbitrary sequence of policies and $\zeta_{1}, \dots \zeta_{n}$ be a sequence of numbers such that $\zeta_{i} \geq 0, \forall i$, $\sum_{i=1}^{n} \zeta_{i} = 1$. Then the policy $\pi'$ defined by:
%    \begin{align}\label{eq:mixture_policy}
%        \pi'(a|s) := \frac{ \sum_{i=1}^{n} \zeta_{i} \, d^{\pi^{(i)}}(s) \pi^{(i)}(a|s)  }{\sum_{i=1}^{n} \zeta_{i} \, d^{\pi^{(i)}}(s) }
%    \end{align}
%    must exist and has the same expectation of producing any state-action pair $(s,a)$ as the right-hand-side of Eq. (\ref{eq:mixture_policy}) when the denominator is nonzero, i.e.,
%    \begin{align*}
%      d^{\pi}(s) \, \pi(a|s) = \sum_{i=1}^{n}d^{\pi^{(i)}}(s) \, \pi^{(i)}(a|s).
%    \end{align*}
% \end{theorem}
For the second scheme, manipulating the interpolated policy is inconvenient since we will have to remember all previous weights and more importantly, the theoretical properties of Boltzmann policies do not hold any longer. 
    To solve this issue, heuristically an information projection step is performed for every interpolated policy to obtain a Boltzmann policy.
    In practice, this policy is found by solving $\min_{\pi} D_{KL}(\pi || \zeta\bar{\pi}_{K+1} + (1-\zeta)\pi_{K})$. 
    Though the information projection step can only approximately guarantee that the CVI bound continues to apply since the replay buffer capacity is finite,  it has been commonly used in practice \cite{haarnoja-SAC2018,Vieillard-2020DCPI}.
In our implementation of deep CPP, the projection problem is solved efficiently using autodifferentiation (Line 7 of Algorithm \ref{alg:deepCPP}).

\subsection{Approximate Interpolation Coefficient}\label{sec:approximate_zeta}

The lowerbound of policy improvement depends on $A^{\pi_{K+1}}_{\pi_{K}, d^{\pi_{K}}}$. 
Though it is general difficult to compute exactly, very recently \cite{Vieillard-2020DCPI} propose to estimate it using batch samples. 
We hence define several quantities following \cite{Vieillard-2020DCPI}: let $B_t$ denote a batch randomly sampled from the replay buffer $B$ and define $\hat{A}_{K}(s) := \max_{a}Q(s, a) - V(s)$ as an estimate of $A^{\tilde{\pi}}_{\pi}(s)$, $\hat{\mathbb{A}}_{K} := \mathbb{E}_{s\sim B}[\hat{A}_{K}(s)]$ as an estimate of $A^{\tilde{\pi}}_{\pi, d^{\pi}}$, and $\hat{A}_{K, \text{min}}:= \min_{s\sim B}\hat{A}_{K}(s)$ as the minimum of the batch. 
When we use linear function approximation with on-policy buffer $B_{K}$, we simply change the minibatch $B$ in the above notations to the on-policy buffer $B_{K}$.

  Given the notations defined above, we can compare the existing interpolation coefficients as the following:

\textbf{CPI}:
    the classic CPI algorithm proposes to use the coefficient:
    \begin{align}
      \zeta_{\textsc{CPI}} = \frac{(1-\gamma) \hat{\mathbb{A}}_{K}}{4 r_{max}},
      \label{eq:cpi_zeta}
    \end{align}
    where $r_{max}$ is the largest possible reward. 
    When the knowledge of the largest reward is not available, approximation based on batches or buffer will have to be employed.

\textbf{Exact SPI}:
    SPI proposes to extend CPI by using the following coefficient:
    \begin{align}
      \zeta_{\textsc{E-SPI}} = \frac{(1-\gamma)^{2} \hat{\mathbb{A}}_{K}}{\gamma\delta\Delta A^{\pi_{K+1}}_{\pi_{K}}},
      \label{eq:espi_zeta}
    \end{align}
    where $\delta, \Delta A^{\pi_{K+1}}_{\pi_{K}}$ were specified in Lemma \ref{thm:SPI}. 
    When $\delta, \Delta A^{\pi_{K+1}}_{\pi_{K}}$ cannot be exactly computed, sample-based approximation will have to employed.

\textbf{Approximate SPI}:
    as suggested by \cite[Remark 1]{pirotta13}, approximate $\zeta$ can be derived if we na\"{i}vely leverage $\delta\Delta A^{\pi_{K+1}}_{\pi_{K}} <\frac{4}{1-\gamma}$:
    \begin{align}
      \zeta_{\textsc{A-SPI}} = \frac{(1-\gamma)^{3} \hat{\mathbb{A}}_{K}}{4\gamma}.
      \label{eq:aspi_zeta}
    \end{align}

\textbf{Linear CPP}:
    if policies are entropy-regularized as indicated in Eq. (\ref{sys_DPPbellman}), we can upper bound $\delta\Delta A^{\pi_{K+1}}_{\pi_{K}}$ by using Lemma \ref{thm:KL}:
    \begin{align}
      \zeta_{\textsc{CPP}} = \frac{(1-\gamma)^{3} \hat{\mathbb{A}}_{K}}{8\gamma C_{K}}.
      \label{eq:linear_cpp}
    \end{align}
    By the definition of $C_{K}$ in Eq. (\ref{CVI_kl}), $\zeta_{\textsc{CPP}}$ can take on a wider range of values than $\zeta_{\textsc{A-SPI}}$.

\textbf{Deep CPI}:
    for better working with deep networks, the following adaptive coefficient was proposed in deep CPI (DCPI) \cite{Vieillard-2020DCPI}:
    \begin{align}
      \begin{split}
      \zeta_{\textsc{DCPI}} = \hat{\zeta_{0}} \frac{m_K}{M_{K}},  \quad \begin{cases}
        m_{K} = \rho_1 m_{K-1} + (1 - \rho_1) \hat{\mathbb{A}}_{K} & \\
        M_{K} = \min(\rho_2 M_{K-1}, \hat{A}_{K, \text{min}}),
      \end{cases}  
    \end{split}
    \label{eq:adaptive_zeta}
    \end{align}
    where $\rho_1, \rho_2 \in (0, 1)$ are learning rates, and $\hat{\zeta}_{0} = \frac{1}{4}$ same with CPI \cite{Kakade02}.

\textbf{Deep CPP}:
    we follow the DCPI coefficient design for making $\zeta_{\textsc{CPP}}$ suitable for deep RL. 
    Specifically, we modify DCPP by defining $\hat{\zeta}_{0} = \frac{1}{C_{K}}$: 
    \begin{align}
      \begin{split}
      \zeta_{\textsc{DCPP}} = \texttt{clip}\left\{\frac{1}{C_K} \frac{m_K}{M_{K}}, \,\, 0, \,\, 1 \right\},
    \end{split}
    \label{eq:cpp_zeta}
    \end{align}
    where $m_K, M_K$ are same as Eq. (\ref{eq:adaptive_zeta}).\\
Based on Eqs. (\ref{eq:linear_cpp}), (\ref{eq:cpp_zeta}), we detail the linear and deep implementations of CPP in the next section.

\subsection{Approximate CPP}\label{sec:CPP_FA}

\begin{algorithm}[t]
  \SetKwInOut{Input}{Input}
  \caption{ Linear Cautious Policy Programming}\label{alg:CPP}
  \Input{$\alpha, \beta, \gamma$ CPP parameters, $I$ the total number of iterations, $T$ the number of steps for each iteration}
  initialize $\theta, \tilde{\pi}_{0}$ at random\;
  empty on-policy buffer ${B}_{K} = \{\}$\;
  \For{$K = 1, \dots, I$}
  {
    \For{$t=1,\dots, T$}{
      Interact using policy $\tilde{\pi}_{K-1}$\;
      Collect $(s^{K}_{t},a^{K}_{t},r^{K}_{t},s^{K}_{t+1})$ into buffer ${B}_{K}$ \;
    }
    %\If{$K \text{ mod } C == 0$}
    {
      compute basis matrix $\Phi_{K}$ using ${B}_{K}$\;
      update $\theta$ by normal equations Eq. (\ref{eq:normal_equations})\;
      compute $\hat{\zeta}_{0} = \frac{1}{C_{K}}$ and $\hat{\zeta} = \hat{\zeta}_{0} \frac{m_{K}}{M_{K}}$ using Eq. (\ref{eq:adaptive_zeta})\;
      empty on-policy buffer ${B}_{K}$\;
    }
  }
 
\end{algorithm}

We introduce the linear implementation of CPP following \cite{lagoudakis2003least,azar2012dynamic} and deep CPP inspired by \cite{Vieillard-2020DCPI} in Algs. \ref{alg:CPP} and \ref{alg:deepCPP}, respectively.
It is worth noting that in linear CPP we assume the interpolated policy $\tilde{\pi}$ is used only for collecting samples (line 5 of Alg. \ref{alg:CPP}) hence no projection is necessary as it does not interfere with computing next policy.

\textbf{Linear CPP}. 
We adopt linear function approximation (LFA) to approximate the Q-function by $Q(s,a) = \phi(s,a)^{T}\theta$, where $\phi(x) = [\varphi_{1}(x), \ldots, \varphi_{M}(x)]^{T}, x=[s,a]^{T}$, $\varphi(x)$ is the basis function and $\theta$ corresponds to the weight vector. One typical choice of basis function is the radial basis function:
\begin{align*}
    \varphi_{i}(x) = \exp\big(-\frac{||x-c_{i}||^{2}}{\sigma^{2}}\big),
\end{align*}
where $c_{i}$ is the center and $\sigma$ is the width. 
We construct basis matrix $\Phi = [\phi_{1}(x_{1}), \ldots, \phi_{M}(x_{N})] \in \mathbb{R}^{T\times M}$, where $T$ is the number of timesteps.  
Specifically, at $K$-th iteration, we maintain an on-policy buffer ${B}_{K}$.
For every timestep $t\in[1,T]$, we collect $(s^{K}_{t},a^{K}_{t},r^{K}_{t},s^{K}_{t+1})$  into the buffer and compute the basis matrix at the end of every iteration.

To obtain the best-fit $\theta_{K+1}$ for the $K + 1$-th iteration, we solve the least-squares problem $||T_{\pi_{K+1}}Q_{\pi_{K}} - \Phi\theta_{K} ||^{2}$:
\begin{align}
    \theta_{K+1} = \big(\Phi^{T}\Phi + \varepsilon I\big)^{-1}\Phi^{T}T_{\pi_{K+1}}Q_{\pi_{K}},
    \label{eq:normal_equations}
\end{align}
where $\varepsilon$ is a small constant preventing singular matrix inversion and ${T_{\pi_{K+1}}Q_{\pi_{K}}}$ is the empirical Bellman operator defined by 
\begin{align*}
  { T_{\pi_{K+1}}Q_{\pi_{K}} } (s^{K}_t, a^{K}_t) := r(s^{K}_t, a^{K}_t) + \gamma \sum_{a} \pi_{K+1}(a|s^{K}_{t+1}) Q_{K}(s^{K}_{t+1},a).
\end{align*}
Since the buffer is on-policy, we empty it at the end of every iteration (line 10).

\begin{algorithm}[t]
  \SetKwInOut{Input}{Input}
  \caption{Deep Cautious Policy Programming}\label{alg:deepCPP}
  \Input{$\alpha, \beta, \gamma$ CPP parameters, $T$ the total number of steps, $F$ the interaction period, $C$ the update period}
  initialize $\theta$ at random\;
  set $\theta^{-} = \theta$, $K = 0$ and buffer ${B}$ to be empty\;
  \For{$K = 1, \dots, T$}
  { interact with the environment using policy $\pi_{\epsilon}$\;
    collect a transition tuple $(s,a,r,s')$ into buffer $\mathcal{B}$ \;
    \If{$K \text{ mod } F == 0$}{
      sample a minibatch $B_{t}$ from ${B}$ and compute the loss $\mathcal{L}_{value}$ and $\mathcal{L}_{policy}$ using Eqs. (\ref{eq:value_loss}), (\ref{eq:policy_loss})\;
      do one step of gradient descent on the loss $\mathcal{L}_{train} = \mathcal{L}_{value} + \mathcal{L}_{policy}$\;
      compute $\hat{A}_{K}, \hat{\mathbb{A}}_{K}$ and moving average $m_{K}, M_{K}$ using Eq. (\ref{eq:adaptive_zeta})\;
    }
    \If{$K \text{ mod } C == 0$}
    {
      $\theta^{-} \leftarrow \theta$ \;
      compute $\hat{\zeta}_{0} = \frac{1}{C_{K}}$ and ${\zeta_{\textsc{CPP}}} = \hat{\zeta}_{0} \frac{m_{K}}{M_{K}}$ using Eq. (\ref{eq:adaptive_zeta})\;
    }
  }
 
\end{algorithm}

% When function approximators are used, several heuristic designs are necessary for scalable implementation of CPP.
% Function approximators could be linear (one layer, linear activation) or multilayer including convolutional layers. 
% Specifically, we describe how to implement CPP leveraging off-policy value iteration architecture such as DQN.
\textbf{Deep CPP. }
Though CPP is an on-policy algorithm, by following \cite{Vieillard-2020DCPI} off-policy data can also be leveraged with the hope that random sampling from the replay buffer covers areas likely to be visited by the policy in the long term.
Off-policy learning greatly expands CPP's coverage, since on-policy algorithms require expensively large number of samples to converge, while  off-policy algorithms are more competitive in terms of sample complexity in deep RL scenarios. 

We implement CPP based on the DQN architecture, where the Q-function is parameterized as $Q_{\theta}$, where $\theta$ denotes the weights of an online network, as can be seen from Line 2.
Line 3 begins the learning loop. 
For every step we interact with the environment using policy $\pi_{\epsilon}$, where $\epsilon$ denotes the epsilon-greedy policy threshold.
As a result, a tuple of experience is collected to the buffer.
% It should be noted that, we can also leverage the interpolated policy $\tilde{\pi}_{K}$ only for interacting with the environment but not for computing the next policy.

  Line 6 of Alg. \ref{alg:deepCPP} begins the update loop.
  We sample a minibatch from the buffer and compute the loss $\mathcal{L}_{value}, \mathcal{L}_{policy}$ defined in Eqs. (\ref{eq:value_loss}), (\ref{eq:policy_loss}), respectively.
  Since our implementation is based on DQN, we do not include additional policy network as done in \cite{Vieillard-2020DCPI}. 
  Instead, we denote the policy as $\pi_{\theta}$ to indicate that the policy is a function of $Q_{\theta}$ as shown in Eq. (\ref{eq:boltzmann_greedy}). 
  The base policy is hence denoted by $\pi^{-}_{\theta}$ to indicate it is computed by the target network of $\theta^{-}$. 
  We define the regression target as:
  \begin{align*}
    \hat{Q}(s_{t},a_{t},r_{t},s_{t+1}) = & ( r_{t} +  \gamma \sum_{a\in\mathcal{A}} \pi_{\theta}(a | s_{t+1})\big( Q^{-}(s_{t+1}, a) + \\
    &  \tau \log \pi_{\theta}(a|s_{t+1}) + \sigma \log\frac{\pi_{\theta}(a|s_{t+1})}{\pi^{-}_{\theta}(a|s_{t+1})} \big).
  \end{align*}
  Hence, the loss for $\theta$ is defined by:
  \begin{align}
    \label{eq:value_loss}
    \mathcal{L}_{value}(\theta) = \mathbb{E}_{(s_{t}, a_{t}, \dots) \sim B}  \left[  \left(  Q_{\theta}(s_{t}, a_{t}) - \hat{Q}(s_{t},a_{t},r_{t},s_{t+1})  \right)^{2}  \right].
  \end{align}
  
  It should be noted that the interpolated policy cannot be directly used as it is generally no longer Boltzmann.
  To tackle this problem, we further incorporate the following minimization problem to project the interpolated policy back to the Boltzmann policy class:
  \begin{align}
    \begin{split}\label{eq:policy_loss}
    &\mathcal{L}_{policy}  (\theta) = \\
    &\mathbb{E}_{(s_{t}, a_{t}, \dots) \sim B}  \left[ D_{KL}  \left( \pi_{\theta}(a_{t}|s_{t}) \,\, \big|\big| \,\, \zeta \mathcal{G}Q_{\theta} + (1 - \zeta) \pi_{\theta}^{-}(a_{t}|s_{t})  \right)\right],
  \end{split}
  \end{align}
  where $\mathcal{G}Q_{\theta}$ takes the maximizer of the action value function. 
  The reason why we can express the policy $\pi$ and $\mathcal{G}Q_{\theta}$ with the subscript $\theta$ is because the policy is a function of action value function, which has a closed-form solution (see \cite{kozunoCVI} for details):
  \begin{align}
    \label{eq:boltzmann_greedy}
    \mathcal{G}Q_{\theta}(a|s) = \frac{ \pi_{\theta}^{-}(a|s)^{\alpha} \exp \left( \beta Q_{\theta}(s, a) \right )}{\sum_{a'\in\mathcal{A}} \pi_{\theta}^{-}(a'|s)^{\alpha} \exp \left( \beta Q_{\theta}(s, a') \right)},
  \end{align}
  which by simple induction can be written completely in terms of $Q_{\theta}$ as $\mathcal{G}Q_{\theta}(a|s) \propto \exp \adaBracket{\sum_{j=0}Q_{\theta_{j}}(s,a)}$ \cite{vieillard2020leverage}.
  Line 8 performs one step of gradient descent on the the compound loss and line 9 computes the approximate expected advantage function for computing $\zeta$.

  % In Alg. \ref{alg:deepCPP} we did not bother to handle the case $A^{{\pi_{K+1}}}_{{\pi_{K}},d^{\pi_{K}}} < 0$: when it is negative we simply do nothing and collect more data in order to lower the variance so there is more chance $A^{{\pi_{K+1}}}_{{\pi_{K}},d^{\pi_{K}}} $ becomes positive.

There is one subtlety in that the definition of $K$ is unclear in the deep RL context: there is no clear notion of \emph{iteration}.
If we na\"{i}vely define $K$ as the the number of steps or the number of updates, then by definition $C_K$ in Eq. (\ref{eq:cpp_zeta}) could quickly converge to 0 or explode, rendering CPP losing the ability of controlling update.
Hence in our implementation, we increment $K$ by one every time we update the target network (every $C$ steps), which results in a suitable magnitude of $K$.

\section{Experimental Results}\label{sec:experimental}

The proposed CPP algorithm can be applied to a variety of entropy-regularized algorithms. In this section, we utilize conservative value iteration (CVI) as the base algorithm in~\cite{kozunoCVI} for our experiments. 
In our implementation, 
for the $K+1$-th update, the baseline policy $\bar{\pi}$ in Eq. (\ref{sys_DPPbellman}) is $\pi_{K}$. %It is worth noting that CVI itself can already \emph{partly} handle policy oscillation, we aim to improve further over CVI. 

For didactic purposes, we first examine all algorithms (specified below) in a safety gridworld and the classic control problem pendulum swing-up.
The tabular gridworld allows for exact computation to inspect  the effect of algorithms.
On the other hand, pendulum swing-up leverages linear function approximation detailed in Alg. \ref{alg:CPP}.
We then apply the algorithms on a set of Atari games to demonstrate the effectiveness of our proposed method.
It is worth noting that even state-of-the-art monotonic improving methods failed in complicated Atari games \cite{papini20-balanceSpeedStabilityPG}.
The gridworld, pendulum swing-up and Atari games manifest the growth of complexity and allow for comparison on how the algorithms trade off stability and scalability.

For the gridworld and pendulum experiments, we compare Linear CPP using coefficient Eq. (\ref{eq:linear_cpp}) against \emph{safe policy iteration} (SPI)~\cite{pirotta13} which is the closest to our work. 
We employ Exact-SPI (E-SPI) coefficient in Eq. (\ref{eq:espi_zeta}) on the gridworld since
in small state spaces where the quantities $\delta, \Delta A^{\pi_{K+1}}_{\pi_{K}}$ can be accurately estimated.
As a result, SPI performance should upper bound that of CPP since CPP was derived by further loosening on SPI.
For problems with larger state-action spaces, SPI performance may become poor as a result of insufficient samples for estimating those quantities, hence Approximate-SPI (A-SPI) Eq. (\ref{eq:aspi_zeta}) should be used. %All three algorithms are examined in both discrete and continuous state spaces.
However, leveraging A-SPI coefficient often results in vanishingly small $\zeta$ values.

For Atari games, we compare Deep CPP leveraging Eq. (\ref{eq:cpp_zeta}) against on- and off-policy state-of-the-art algorithms, see Section \ref{experiment:atari} for a detailed list. Specifically, we implement deep CPP using off-policy data to show it is capable of leveraging off-policy samples, hence greatly expanding its coverage since on-policy algorithms typically have expensive sample requirement.

%Due to the page limit, we present the detailed experimental settings in the Appendix.

%One might gain intuitive understanding of monotonic improvement from Figure (\ref{gridworld_oscillation}) that illustrates the norm-2 and norm-$\infty$ \emph{policy improvement oscillation} which are defined as:

%It is apparent that the magnitude of oscillation value is much alleviated by CPP compared to CVI, and closely resembles that of SPI, validating effectiveness of CPP.
\subsection{Gridworld with Danger}\label{apdx:gridworld}

\subsubsection{Experimental Setup}
\iffalse
\begin{figure}[t]
  \begin{subfigure}[]{0.475\linewidth}
    \includegraphics[width=\textwidth]{gridworld_cumuR.pdf}
    \caption{Cumulative reward}
    \label{comp_cr}
  \end{subfigure}
  \begin{subfigure}[]{0.475\linewidth}
    \includegraphics[width=\textwidth]{refactored_gridworld_oscillation.pdf}
    \caption{Values of policy oscillation}
    \label{gridworld_oscillation}
  \end{subfigure}
  \caption{Comparison between SPI, CPP, and CVI on the safety grid world. The black line shows the mean SPI cumulative reward, the blue line CPP, and the red line CVI in (\ref{comp_cr}), with the shaded area indicating $\pm 1$ standard deviation. (\ref{gridworld_oscillation}) compares the respective policy oscillation value defined in Eq. (\ref{oscillation_measure}).}
  \label{gridworld-rci}
\end{figure}
\fi

%The external policy $\hat{\pi}$ also has the knowledge of danger zones but not the desired states and the obstacles. Hence the expected advantage $A_{\pi}^{\hat{\pi}}$ might quickly become negative as learning proceeeds. To solve this problem the following scheme of $\zeta$ is adopted:

 The agent in the $5 \times 5$ grid world starts from a fixed position at the upper left corner and can move to any of its neighboring states with success probability $p$ or to a random different direction with probability $1-p$. 
 Its objective is to travel to a fixed destination located at the lower right corner and receives a $+1$ reward upon arrival. 
 Stepping into two danger grids located at the center of the gridworld incurs a cost of $-1$. Every step costs $-0.1$. 
 We maintain tables for value functions to inspect the case when there is no approximation error. 
 Parameters are tuned to yield empirically best performance. For testing the \emph{sample efficiency}, every iteration terminates after 20 steps or upon reaching the goal, and only 30 iterations are allowed for training. For statistical significance, the results are averaged over 100 independent trials.

 \subsubsection{Results}

 Figure (\ref{comp_cr}) shows the performance of SPI, CPP, and CVI, respectively. 
 Recall that SPI used the \emph{exact} coefficient Eq. (\ref{eq:espi_zeta}).
 The black, blue, and red lines indicate their respective cumulative reward ($y$-axis) along the number of iterations ($x$-axis). The shaded area shows $\pm 1$ standard deviation. CVI learned policies that visited danger regions more often and result in delayed convergence compared to CPP. 
Figure (\ref{gridworld_oscillation}) compares the average policy oscillation defined in Eq. (\ref{oscillation_measure}).

\begin{figure}[t]
  \begin{minipage}{.5\linewidth}
  \centering
  \subfloat[]{\label{comp_cr}\includegraphics[width=\linewidth]{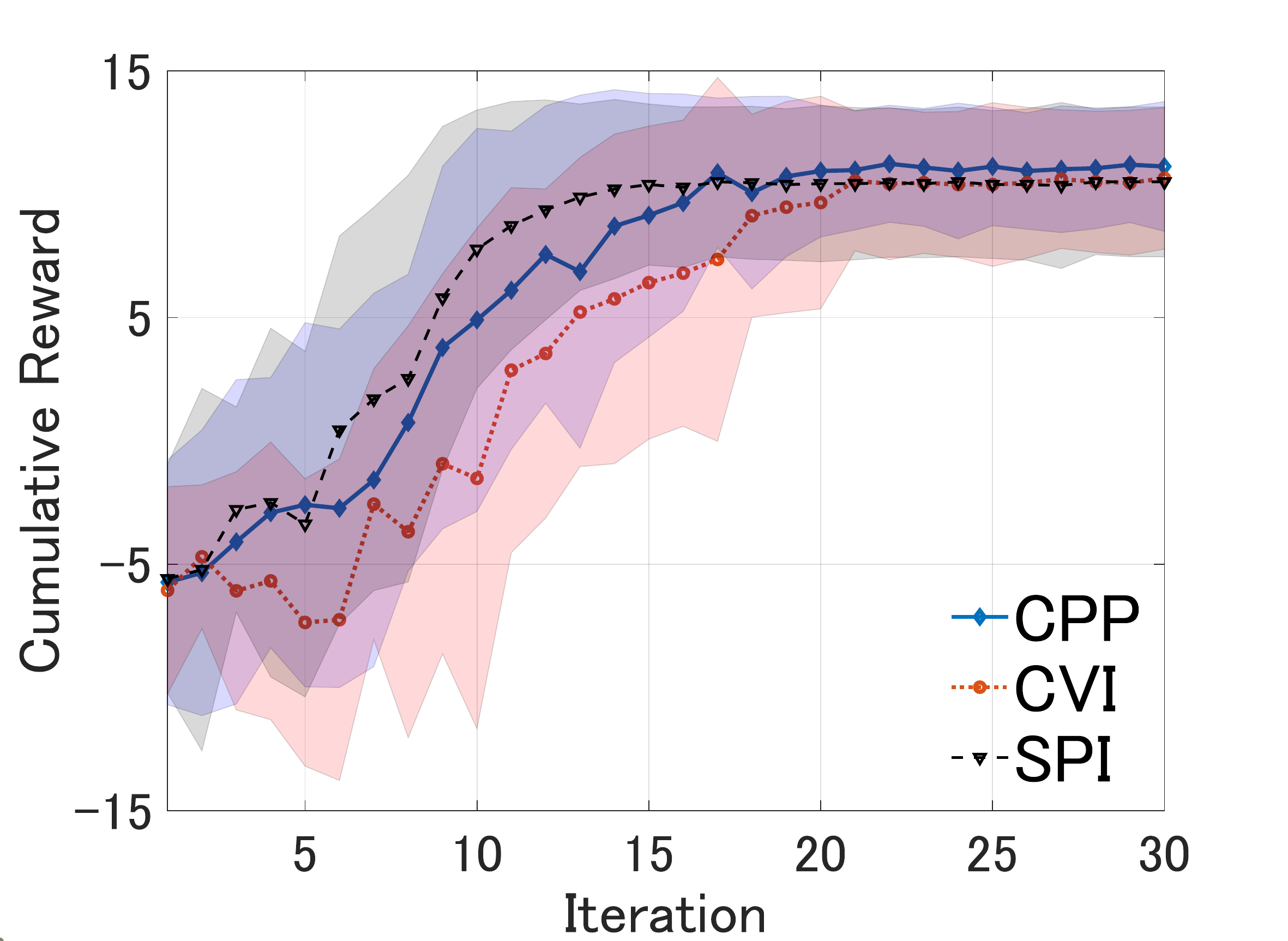}}
  \end{minipage}%
  \begin{minipage}{.5\linewidth}
  \centering
  \subfloat[]{\label{gridworld_oscillation}\includegraphics[width=\linewidth]{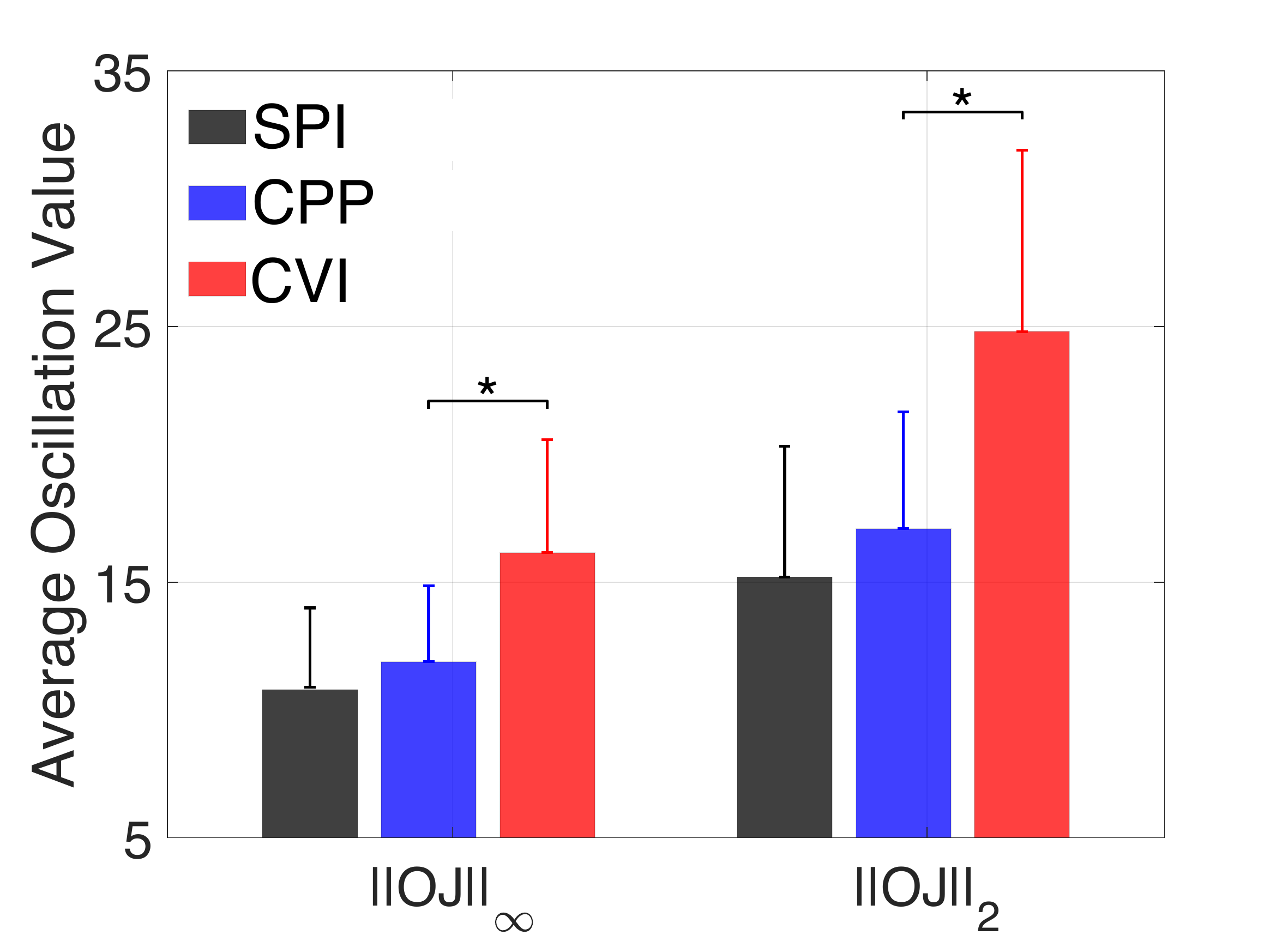}}
  \end{minipage}
  \caption{Comparison between SPI, CPP, and CVI on the safety grid world. The black line shows the mean SPI cumulative reward, the blue line CPP, and the red line CVI in Figure (\ref{comp_cr}), with the shaded area indicating $\pm 1$ standard deviation. 
  Figure (\ref{gridworld_oscillation}) compares the respective policy oscillation value defined in Eq. (\ref{oscillation_measure}).
  }
  \label{gridworld_rci}
\end{figure}

The slightly worse oscillation value of CPP than SPI with $\zeta_{\textsc{E-SPI}}$ is expected as CPP exploited a lower bound that is looser than that of SPI. 
However, as will be shown in the following examples when both linear and nonlinear function approximation are adopted, SPI failed to learn meaningful behaviors due to the inability to accurately estimate the complicated lower bound.

\subsection{Pendulum Swing Up}\label{experiment:pendulum}

%In this section we examine all algorithms on simulated pendulum swing-up, a classical control problem with continuous state space. 
Since the state space is continuous in the pendulum swing up, E-SPI can no longer expect to accurately estimate $\delta\Delta A^{\pi_{K+1}}_{\pi_{K}}$, so we employ A-SPI in Eq. (\ref{eq:aspi_zeta}) and compare both E-SPI and A-SPI against Linear CPP Eq. (\ref{eq:linear_cpp}).

\subsubsection{Experimental Setup}\label{sec:pendulum}
%\subsubsection{Experimental Setting}

A pendulum of length $1.5$ meters has a ball of mass $1$kg at its end starting from the fixed initial state $[0, -\pi]$. The pendulum attempts to reach the goal $[0, \pi]$ and stay there for as long as possible. The state space is two-dimensional $s=[\theta, \dot{\theta} ]$, where $\theta$ denotes the vertical angle and $\dot{\theta}$ the angular velocity. Action is one-dimensional torque $[-2, 0, 2]$ applied to the pendulum. The reward is the negative addition of two quadratic functions quadratic in angle and angular velocity, respectively:
\begin{align*}
    R = - \frac{1}{z}(a\theta^{2} - b\dot{\theta}^{2}),
\end{align*}
where $\frac{1}{z}$ normalizes the rewards and a large $b$ penalizes high angular velocity. We set $z = 10, a=1, b=0.01$.

%Since the state space is continuous, function approximation has to be employed. 
% We adopt linear function approximation (LFA) to approximate the Q-function by $Q(s,a) = \phi(s,a)^{T}\theta$, where $\phi(x) = [\varphi_{1}(x), \ldots, \varphi_{M}(x)]^{T}, x=[s,a]^{T}$, $\varphi(x)$ is the basis function and $\theta$ corresponds to the weight vector. One typical choice of basis function is the radial basis function:
% \begin{align*}
%     \varphi_{i}(x) = \exp\big(-\frac{||x-c_{i}||^{2}}{\sigma^{2}}\big),
% \end{align*}
% where $c$ is the center and $\sigma$ is the width. We construct $\Phi = [\phi_{1}(x_{1}), \ldots, \phi_{M}(x_{N})] \in \mathbb{R}^{N\times M}$. We set $M = 100$ for all algorithms. To obtain the best-fit $\theta_{K+1}$ for the $K + 1$-th iteration, we solve the least-squares problem $||T_{\pi_{K+1}}Q_{\pi_{K}} - \Phi\theta_{K} ||^{2}$:
% \begin{align*}
%     \theta_{K+1} = \big(\Phi^{T}\Phi + \alpha I\big)^{-1}\Phi^{T}T_{\pi_{K+1}}Q_{\pi_{K}},
% \end{align*}
% where $\alpha$ is a small constant preventing singular matrix inversion and $T_{\pi_{K+1}}Q_{\pi_{K}}$ is the empirical Bellman operator.

To demonstrate that the proposed algorithm can ensure monotonic improvement even with a small number of samples, we allow 80 iterations of learning; each iteration comprises 500 steps. For statistical evidence, all figures show results averaged over 100 independent experiments.
%Detailed experimental setting is given in \ref{apdx:pendulum}.

\iffalse
\begin{figure}[tbp]
  \begin{subfigure}[]{0.475\linewidth}
    \centering
    \includegraphics[width=\linewidth]{refactored_pendulum_AE_Oscillation.pdf}
    \caption{Values of policy oscillation}
    \label{pendulum_oscillation}
  \end{subfigure}%
  \begin{subfigure}[]{0.475\linewidth}
    \centering
    \includegraphics[width=\linewidth]{refactored_zeta_comparison.pdf}
    \caption{Values of $\zeta$}
    \label{pendulum_zeta}
  \end{subfigure}\\[1ex]
  \hfill
  \begin{subfigure}[]{0.475\linewidth}
    \centering
    \includegraphics[width=\linewidth]{pendulum_reward_comare_expY.pdf}
    \caption{xssx}
    \label{pendulum_reward}
  \end{subfigure}
  \caption{Comparison of SPI, CPP, and CVI on the pendulum swing up task.
  (a) Illustrates the policy oscillation value defined in Eq. (\ref{oscillation_measure}).
  (b) Shows the $\zeta$ values.}
  \label{pendulum_results}
\end{figure}
\fi

\begin{figure}
  \begin{minipage}{.5\linewidth}
  \centering
  \subfloat[]{\label{pendulum_oscillation}\includegraphics[width=\linewidth]{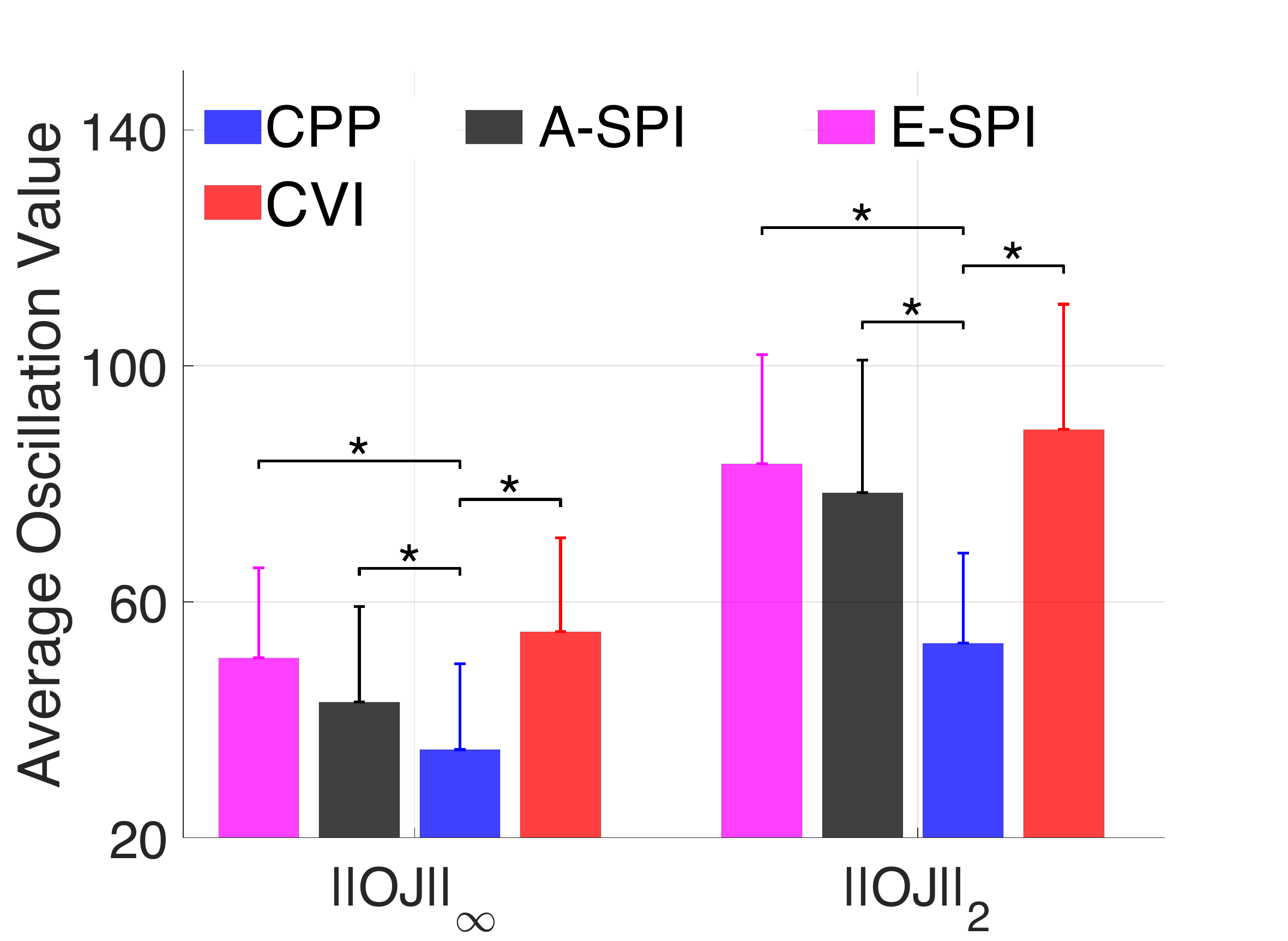}}
  \end{minipage}%
  \begin{minipage}{.55\linewidth}
  \centering
  \subfloat[]{\label{pendulum_reward}\includegraphics[width=\linewidth]{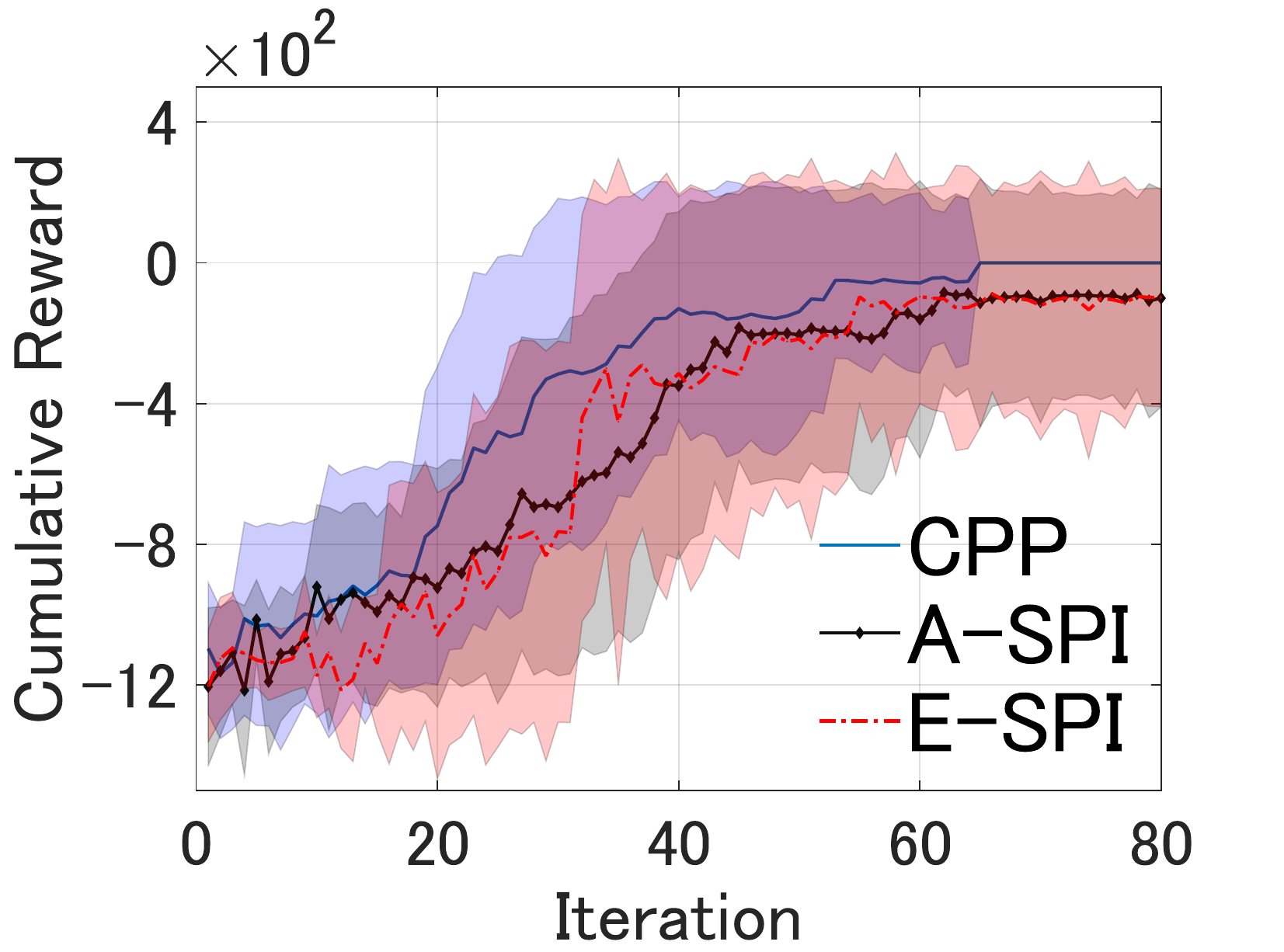}}
  \end{minipage}\par\medskip
  \centering
  \begin{minipage}{.55\linewidth}
    \centering
    \subfloat[]{\label{pendulum_zeta}\includegraphics[width=\linewidth]{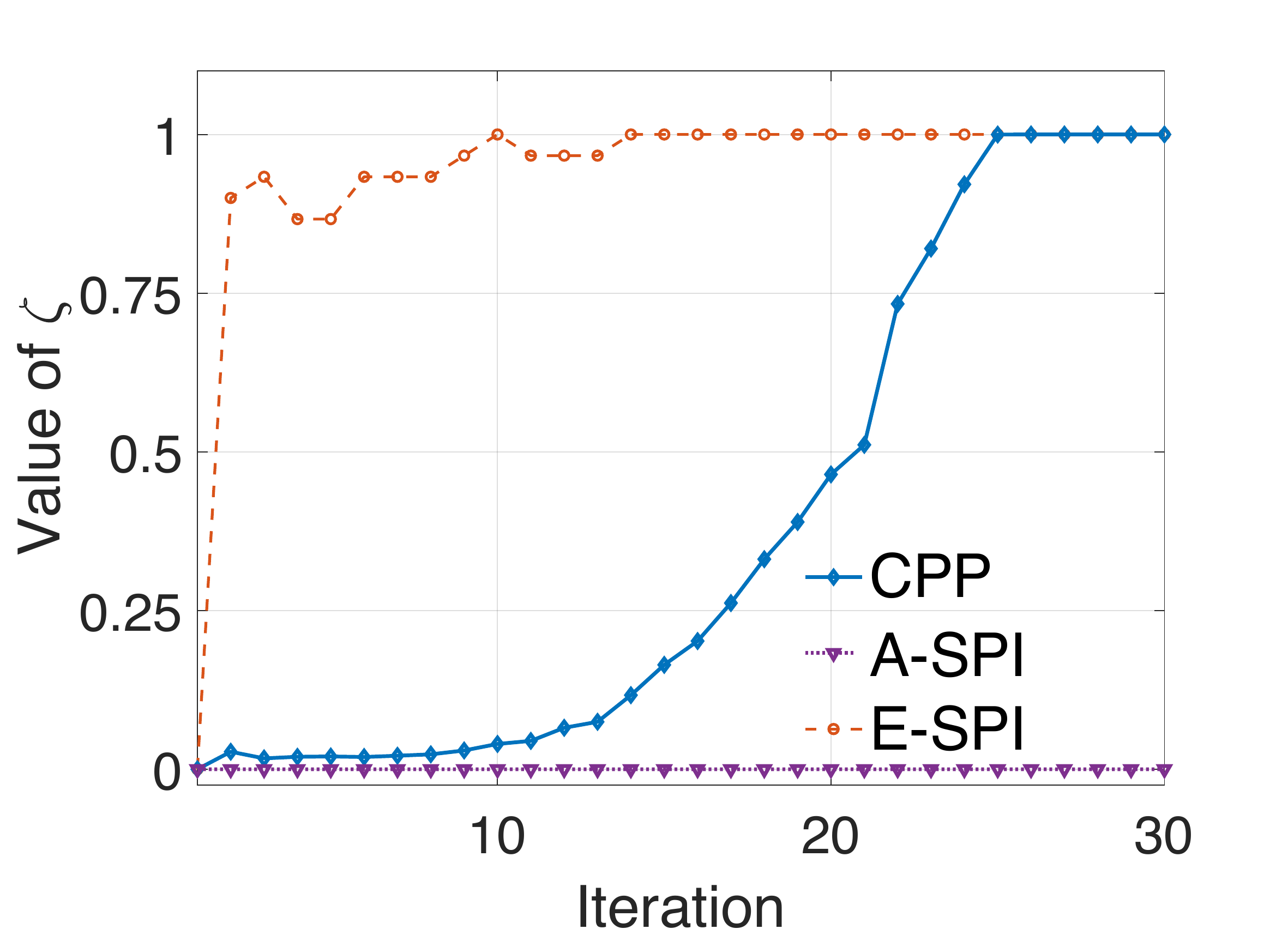}}
    \end{minipage}%
  \caption{Comparison of SPI, CPP, and CVI on the pendulum swing up task.
  Figure (\ref{pendulum_oscillation}) illustrates the policy oscillation value defined in Eq. (\ref{oscillation_measure}).
  Figure (\ref{pendulum_reward}) shows the cumulative reward with $\pm 1$ standard deviation.
  Figure (\ref{pendulum_zeta}) shows the $\zeta$ values. 
  }
  \label{pendulum_results}
\end{figure}

\begin{figure}[h]
    \begin{subfigure}[t]{.99\linewidth}
    \centering
    \includegraphics[width=\linewidth]{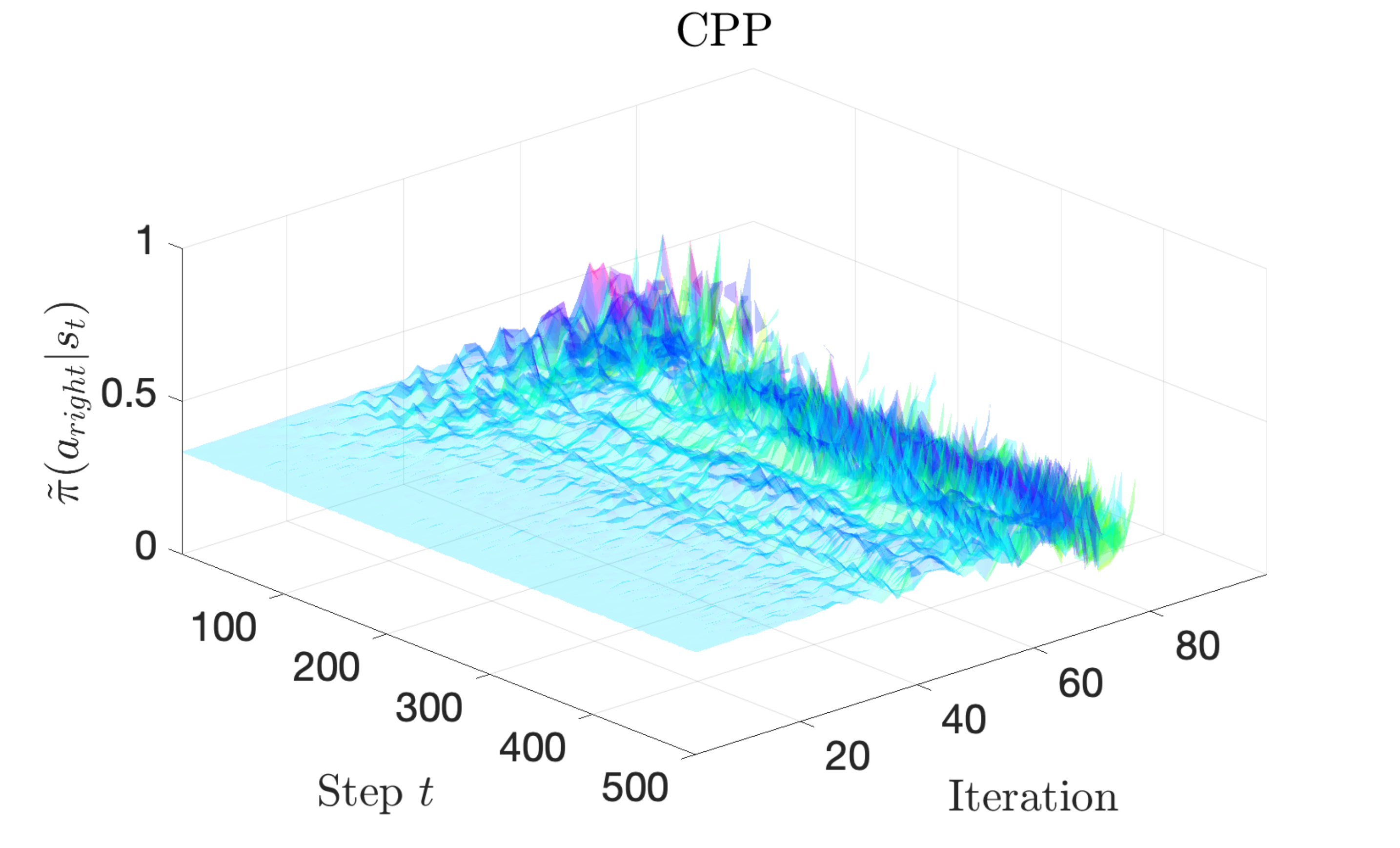}
    \caption{CPP interpolated policy of swinging right $\tilde{\pi}(a_{right}|s_{t})$.}
    \label{fig:cpp_prob}
    \end{subfigure}\\
\vfill
    \begin{subfigure}[t]{.99\linewidth}
      \centering
      \includegraphics[width=\linewidth]{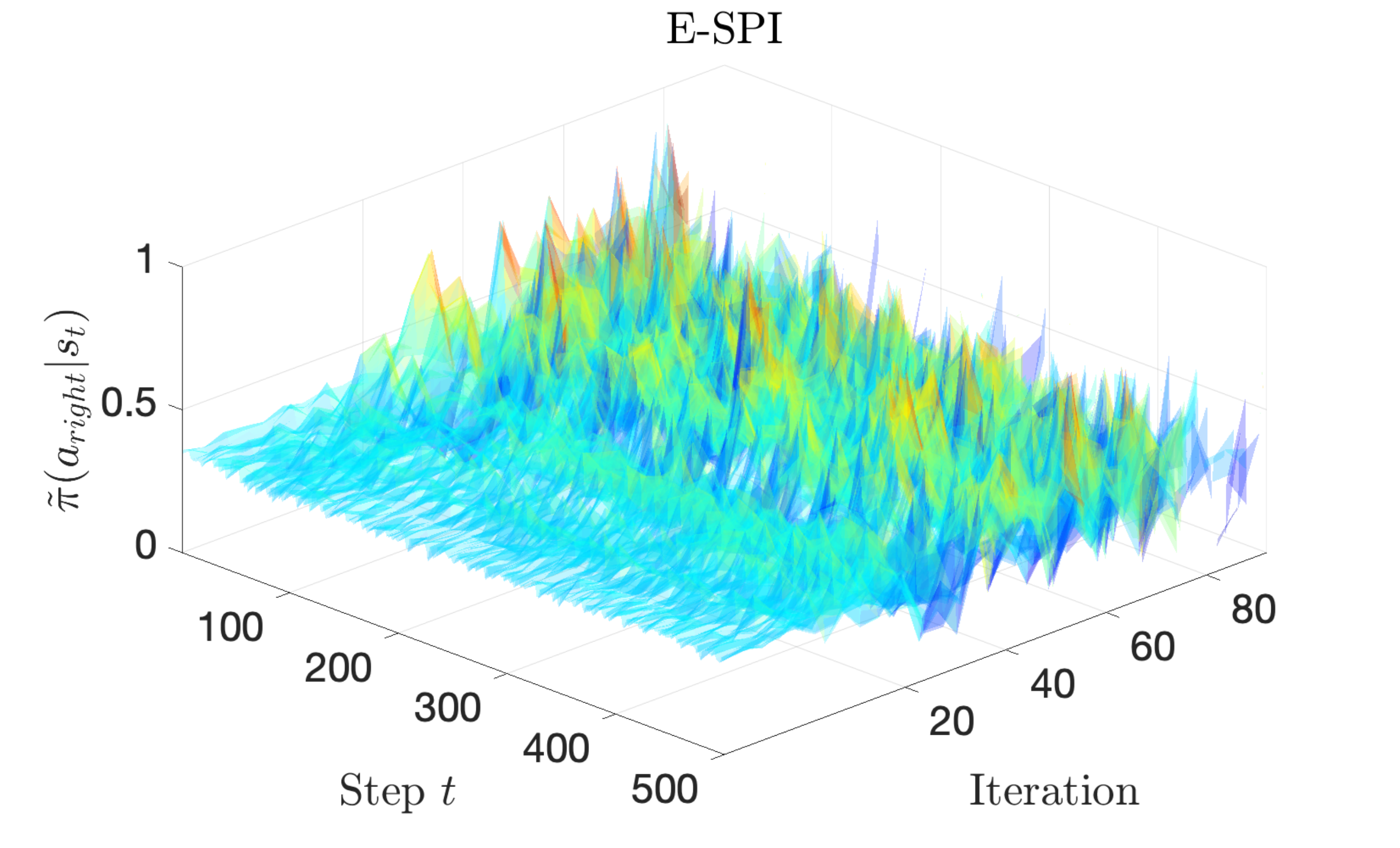}
      \caption{E-SPI interpolated policy of swinging right $\tilde{\pi}(a_{right}|s_{t})$.}
      \label{fig:espi_prob}
      \end{subfigure}
    \caption{CPP and E-SPI interpolated policies of pendulum swinging right $\tilde{\pi}(a_{right}|s_{t})$ ($z$-axis) for timesteps $t=1,\dots,500$ ($x$-axis) from the first to last iteration ($y$-axis).
    E-SPI interpolated policy performed might much more aggressive than the CPP policy caused by the large $\zeta$ values shown in Figure (\ref{pendulum_zeta}).
    }
    \label{pendulum_probability_change}
\end{figure}

\begin{figure*}[t]
  \includegraphics[width=0.95\textwidth]{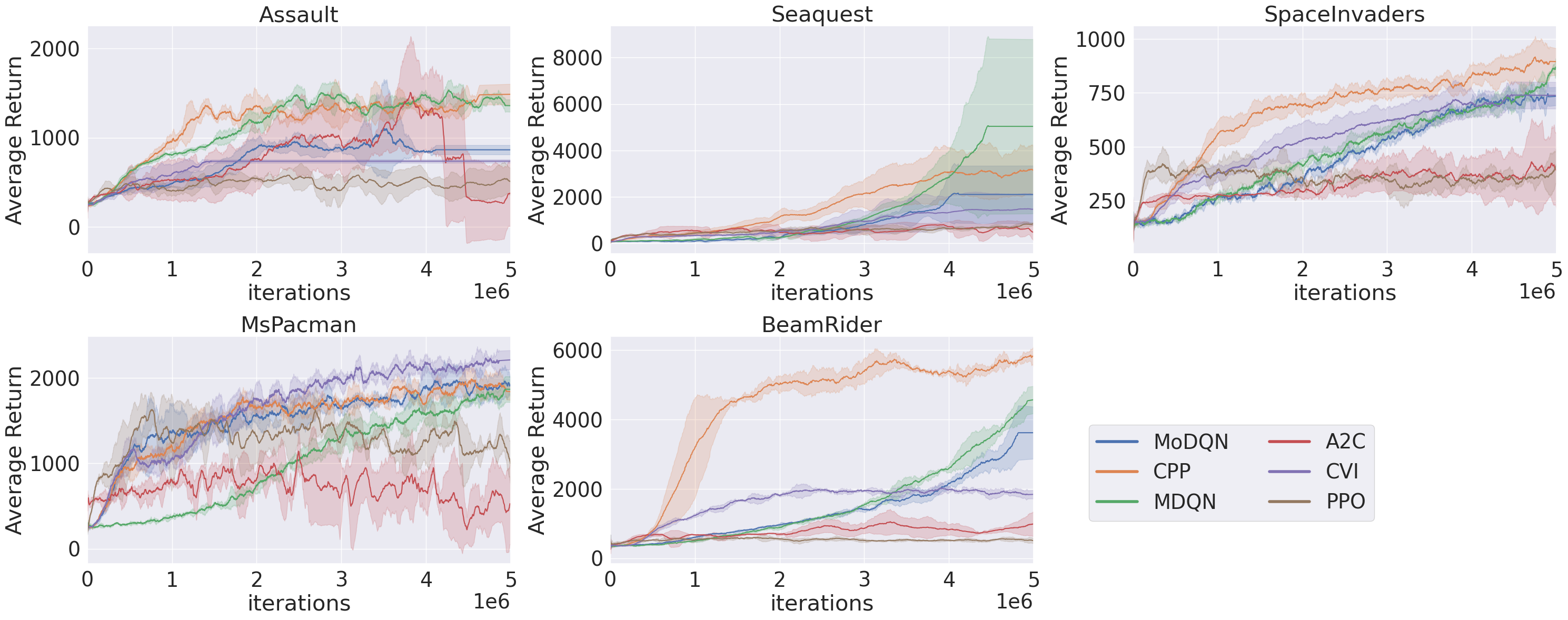}
\caption{Comparison on Atari games averaged over 3 random seeds. 
CPP, MoDQN, MDQN and CVI are implemented as variants of DQN and hence are off-policy. 
PPO and A2C are on-policy.
Correspondence between algorithms and colors is shown in the lower right corner.
Overall, CPP achieved the best balance between final scores, learning speed and oscillation values. 
}
\label{atari}
\end{figure*}

\subsubsection{Results}

We compare CPP with CVI and both E-SPI and A-SPI in Figure \ref{pendulum_results}.
In this simple setup, all algorithms showed similar trend.
But CPP managed to converge to the optimal solution in all seeds, as can be seen from the variance plot.
On the other hand, both SPI versions exhibited lower mean scores and large variance, which indicate that for many seeds they failed to learn the optimal policy.
In Figure (\ref{pendulum_oscillation}), both E-SPI and CVI exhibited wild oscillations, resulting in large average oscillaton values, in which the oscillation criterion is defined as:
\begin{align}
  \begin{split}
     &\forall K, \text{ s.t. }  R_{K+1} - R_{K} < 0, \\
 	& ||\mathcal{O}J||_{\infty} = \max_{K} |R_{K+1} - R_{K}|, \\
     &||\mathcal{O}J||_{2} =  \sqrt{\big(\sum_{K} (R_{K+1} - R_{K})^{2}\big)}, 
  \end{split}
     \label{oscillation_measure}
 \end{align}
 where $R_{K+1}$ refers to the cumulative reward at the $K+1$-th iteration. It is worth noting that the difference $R_{K+1} - R_{K}$ is obtained by $\tilde{\pi}_{K+1}, \tilde{\pi}_{K}$, which is the lower bound of that by $\tilde{\pi}_{K+1}, {\pi}_{K}$. Intuitively, $|| \mathcal{O}J ||_{\infty}$ and $|| \mathcal{O}J ||_{2}$ measure \emph{maximum} and \emph{average} oscillation in cumulative reward. The stars between CPP and CVI represent statistical significance at level $p=0.05$. %In this problem SPI is tractable, hence SPI upper bounds both the cumulative reward and oscillation value of CPP. However, the difference between SPI and CPP oscillation value is insignificant, suggesting the proposed algorithm is equally effective for small-scale problems as with SPI. 

%A-SPI was overly conservative, hence achieved smaller $||\mathcal{O}J||_{\infty}$ but not the $||\mathcal{O}J||_{2}$. On the other hand, CPP learned smoothly thanks to the smooth growth of $\zeta$ and had significant less oscillation value than the others. The stars represent statistical significance at level $p=0.05$. %The drastic behavior of SPI comes from insufficient samples and subsequent inaccurate estimates of Eq. (\ref{J_first_exact}). SPI picks very large $\zeta$ even from the beginning of learning forming a sharp contrast to CPP.

%The drastic behavior of SPI comes from the huge gap between exact and approximate SPI: in the E (exact) version, insufficient samples led to vanishingly small values of $\delta$ and $\Delta A^{\pi_{K+1}}_{\pi_{K}}$ and hence large $\zeta$ in Figure (\ref{pendulum_imp}). The aggressive choice of $\zeta$ led to large oscillation value. On the other hand, A-SPI (approximate) went to the other extreme of producing vanishing $\zeta$ due to the loose bound $\delta \Delta A^{\tilde{\pi}}_{\pi} \leq \frac{4}{1-\gamma}$, as is obvious from the almost horizontal lines in the same figures: A-SPI had average value $\Delta J^{\tilde{\pi}_{K+1}}_{\pi_{K}, d} = 2.39\times 10^{-9}$ and $\zeta = 1.69\times 10^{-6}$.

The reason for SPI's drastic behavior can be observed in Figure (\ref{pendulum_zeta}) (truncated to 30 iterations for better view); in E-SPI, insufficient samples led to very large $\zeta$. The aggressive choice of $\zeta$ led to a large oscillation value. On the other hand, A-SPI went to the other extreme of producing vanishingly small $\zeta$ due to the loose choice of $\zeta$ for ensuring improvement of $\Delta J^{\pi'}_{\pi, d^{\pi'}} \geq \frac{(1-\gamma)^{3} (A^{\tilde{\pi}}_{\pi})^{2}}{8\gamma}$, as can be seen from the almost horizontal lines in the same figure; A-SPI had average value $\Delta J^{\tilde{\pi}_{K+1}}_{\pi_{K},d^{\tilde{\pi}_{K+1}}} = 2.39\times 10^{-9}$ and $\zeta = 1.69\times 10^{-6}$. 
CPP converged with much lower oscillation thanks to the smooth growth of the $\zeta$ values; CPP was cautious in the beginning ($\zeta \approx 0$) and gradually became confident in the updates when it was close to the optimal policy ($\zeta \approx 1$).

However, it might happen that $\zeta$ values are large but probability changes are actually small and vice versa.
To certify CPP did not produce such pathological mixture policy and indeed cautiously learned, we plot in Figure \ref{pendulum_probability_change} the interpolated policies of CPP and E-SPI yielding action probability of the pendulum swinging right $\tilde{\pi}(a_{right}|s_{t})$. 
The probability change is plotted in $z$-axis, timesteps $t = 1,\dots, 500$ of all iterations are drawn on $x,y$ axes.
For both cases, $\tilde{\pi}(a_{right}|s) \approx 0.33$ which is uniform at the beginning of learning. 
However, E-SPI policy $\tilde{\pi}(a_{right}|s)$ gradually peaked from around 10th iteration, which led to very aggressive behavior policy.
Such aggressive behavior was consistent with the overly large $\zeta$ values shown in Figure (\ref{pendulum_zeta}).
On the other hand, CPP policy $\tilde{\pi}(a_{right}|s)$ was more tempered and showed a gradual change conforming to its $\zeta$ change.
The probability plots together with $\zeta$ values in Figure (\ref{pendulum_zeta}) indicate that the CPP interpolation was indeed effective in producing non-trivial diverse mixture policies.

{\color{black}
\subsection{Atari Games}\label{experiment:atari}

\subsubsection{Experimental Setup}\label{sec:atari_setup}

We applied the algorithms to a set of challenging Atari games: \texttt{MsPacmann}, \texttt{SpaceInvaders}, \texttt{Beamrider}, \texttt{Assault} and \texttt{Sea\-quest} \cite{bellemare13-arcade-jair} using the adaptive $\zeta$ introduced in Eq. (\ref{eq:cpp_zeta}). 
We compare deep CPP with both on- and off-policy algorithms to demonstrate that CPP is capable of achieving superior balance between learning speed and oscillation values. 

For on-policy algorithms, we include the celebrated proximal policy gradient (PPO) \cite{schulman2017proximal}, a representative trust-region method.
We also compare with Advantage Actor-Critic (A2C) \cite{mnih2016asynchronous} which is a standard on-policy actor-critic algorithm: our intention is to confirm the expensive sample requirement of on-policy algorithms typically render them underperformant when the number of timesteps is not sufficiently large.

For the off-policy algorithms, we decide to include several state-of-the-art DQN variants: Munchausen DQN (MDQN) \cite{vieillard2020munchausen} features the implicit KL regularization brought by the Munchausen log-policy term: it was shown that MDQN was the only non-distributional RL method outperforming distributional ones.
We also include another state-of-the-art variant: Momentum DQN (MoDQN) \cite{Vieillard2020Momentum} that avoids estimating the intractable base policy in KL-regularized RL by constructing momentum. MoDQN has been shown to obtain superior performance on a wide range of Atari games.
Finally, as an ablation study, we are interested in the case $\zeta = 1$, which translates to conservative value iteration (CVI) \cite{kozunoCVI} based on the framework Eq. (\ref{sys_DPPbellman}). 
CVI has not seen deep RL implementation to the best of our knowledge. Hence a performant deep CVI implementation is of independent interest.

All algorithms are implemented using library Stable Baselines 3 \cite{stable-baselines3}, and tuned using the library Optuna \cite{optuna}. Further, all on- and off-policy algorithms share the same network architectures for their group (i.e. MDQN and CPP share the same architecture and PPO and A2C share another same architecture) for fair comparison. The experiments are evaluated over 3 random seeds.  Details are provided in \ref{apdx:atari}.
We expect that on simple tasks PPO and A2C might be stable due to the on-policy nature, but too slow to learn meaningful behaviors.
However, PPO is known to take drastic updates and heavily needs code-level optimization to correct the drasticity \cite{engstrom2020implementation}.
On the other hand, for complicated tasks, too drastic policy updates might be corrupted by noises and errors, leading to divergent learning.
By contrast, CPP should balance between learning speed and oscillation value, leading to gradual but smooth improvement.

\begin{table*}[t!]
  \centering
  \begin{tabular}{|c|c|c|c|c|c|c|}
    \hline
    Criterion & Algorithm &\texttt{Assault} & \texttt{Seaquest}  & \texttt{SpaceInvaders} & \texttt{MsPacman} & \texttt{BeamRider}\\
    \hline
    \multirow{6}{5em}{\,\, $||\mathcal{O}J||_{2}$}& CPP& 151 & 622 & 89 & 249 & 460 \\
     &MDQN& 129 & 2149 & 77 & 202 & 220 \\
     &MoDQN&162& 813 & 91 & 288 & 718 \\
     &CVI& 77  & 449  & 83 & 292 & 220 \\
     &PPO&  74 & 68  & 72 & 280 & 74 \\
     &A2C& 218 & 98  & 48 & 395 & 87\\
    \hline
    \multirow{6}{5em}{\,\, $||\mathcal{O}J||_{\infty}$} &CPP& 59 & 561 & 42 & 26 & 292\\
    &MDQN& 51 & 2141 & 16 & 52 & 149 \\
    &MoDQN&111& 716 & 36 & 124 & 665 \\
    &CVI& 6  & 361 & 51 & 98 & 105 \\
    &PPO& 16 & 9  & 7 & 36 & 33 \\
    &A2C& 52 & 15  & 8 & 249 & 34\\
   \hline
  \end{tabular}
  \caption{The oscillation values of algorithms listed in Sec. \ref{sec:atari_setup} measured in $||\mathcal{O}J||_{2}$ and $||\mathcal{O}J||_{\infty}$ defined by Eq. (\ref{oscillation_measure}). 
    CPP achieved the best balance between final score, learning speed and oscillation values. 
    Note that CPP was implemented to leverage off-policy data.
    Algorithms of small oscillation values, such as PPO, failed to compete with CPP in terms of final scores and convergence speed.
  }
  \label{tab:oscillation_comparison}
\end{table*}

\subsubsection{Results}

\textbf{Final Scores. }
As is visible from Figure \ref{atari}, Deep CPP achieved either the first or second place in terms of final scores on all environments, with the only competitive algorithm being MDQN which is the state-of-the-art DQN variant,  and occasionally CVI which is the case of $\zeta \!=\! 1$.
However, MDQN suffered from numerical stability on the environment \texttt{Seaquest} as can be seen from the flat line at the end of learning.

CVI performed well on the simple environment \texttt{MsPacman}, which can be interpreted as that learning on simple environments is not likely to oscillate, and hence the policy regularization imposed by $\zeta$ is not really necessary, setting $\zeta \!=\! 1$ is the best approach for obtaining high return. 
However, in general it is better to have adjustable update: on the environment \texttt{BeamRider} the benefit of adjusting the degree of updates was significant: CPP learning curve quickly rised at the beginning of learning, showing a significant large gap with all other algorithms.
Further, while CVI occasionally performed well, it suffered also from numerical stability: on the environment \texttt{Assault}, CVI and MoDQN achieved around 1000 final scores but ran into numerical issues as visible from the end of learning.
This problem has been pointed out in \cite{Vieillard2020Momentum}. 

On the other hand, on all environments on-policy algorithms A2C and PPO failed to learn meaningful behaviors. 
On some environment such as \text{Assault} A2C showed divergent learning behavior at around $4 \times 10^{6}$ and PPO did not learn meaningful behavior until the end.
This observation suggests that the sample complexity of on-policy algorithms is high and generally not favorable compared to off-policy algorithms.

\begin{figure}[t]
  \centering
  \includegraphics[width=\linewidth]{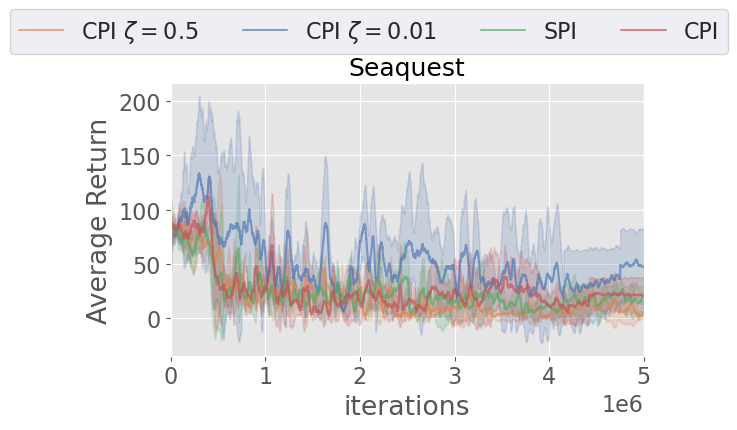}
  \caption{Learning curves of DCPI on \texttt{Seaquest} with four coefficient designs. 
  All designs achieved the final score of 50, while CPP achieved around 3000 in Figure \ref{atari}.
  }
  \label{fig:ablation}
\end{figure}

\textbf{Oscillation. }
The averaged oscillation values of all algorithms are listed in Table \ref{tab:oscillation_comparison}.
While MDQN showed competitive performance against CPP, it exhibited wild oscillation on the difficult environment \texttt{Seaquest} \cite{azar18-noisynet} and finally ran into numerical issue as indicated by the flatline near the end. 
The oscillation value reached to around 2100. Since MDQN is the state-of-the-art regularized value iteration algorithm featuring implicit regularization, this result illustrates that on difficult environments, only reward regularization might not be sufficient to maintain stable learning. 
On the other hand, CPP achieved a balance between stable learning and small oscillation, with oscillation value around $600$, attaining final score slightly lower than MDQN and higher than MoDQN and CVI.

The oscillation values and final scores should be combined together for evaluating how algorithms perform.
CVI, MoDQN sometimes showed similar performance to CPP, but in general the final scores are lower than CPP, with higher oscillation values. 
On the other hand, MDQN showed competitive final scores, but sometimes it exhibited wild oscillation and ran into numerical issues,  implying that on some environments where low oscillation is desired, CPP might be more desirable than MDQN.
On-policy algorithms even showed low oscillation values, but their final scores are considered unacceptable.

\subsection{Ablation Study}\label{sec:ablation}

We are interested in comparing the performance of DCPI with CPP to see the role played by $\zeta_{\textsc{DCPP}}$.
It is also enlightening by inspecting the result of fixing $\zeta$ as a constant value.
In this subsection, we perform ablation study by comparing the the following four designs:
\begin{itemize}
  \item DCPI with fixed $\zeta \!=\! 0.01$: this is to inspect the result of constantly low interpolation coefficient.
  \item DCPI with fixed $\zeta \!=\! 0.5$: this is to examine the performance of equally weighting all policies.
  \item CPI: this uses the coefficient from Eq. (\ref{eq:cpi_zeta}). 
  \item SPI: this uses the DCPI architecture, but we compute $\zeta_{\textsc{A-SPI}}$ by using Eq. (\ref{eq:aspi_zeta}).
\end{itemize}
We examine those four designs on the challenging environment \texttt{Seaquest}.
Other experimental settings are held same with Sec. \ref{experiment:atari}.

As can be seen from Figure \ref{fig:ablation}, all designs showed a similar trend of converging to some sub-optimal policy.
The final scores were around 50, which was significantly lower than CPP in Figure \ref{atari}.
This result is not surprising since for $\zeta=0.01$, almost no update was performed. 
For $\zeta=0.5$, the algorithm weights contribution of all policies equally without caring about their quality.
On this environment, $\zeta_{\text{A-SPI}}$ is vanishingly small similar with that shown in Figure (\ref{pendulum_zeta}).
Lastly, for CPI the number of learning steps is not sufficient for learning meaningful behavior.

\section{Discussion}\label{sec:discussion}

Leveraging the entropy-regularized formulation for monotonic improvement has been recently analyzed in the \emph{policy gradient} literature for tabular MDP \cite{mei20b-globalConvergence,Agarwal20-theoryPolicyGradient,cen2020fastGlobalConvergence}.
In the tabular MDP setting with exact computation, monotonic improvement and fast convergence can be proved. 
However, realistic applications are beyond the scope for their analysis and no scalable implementation has been provided.
On the other hand, value-based methods have readily applicable error propagation analysis \cite{Munos2005,scherrer15-AMPI,Lazaric2016-classificationPolicyIterationAnalysis} for the function approximation setting, but they seldom focus on monotonic improvement guarantees such as $ J^{K+1}-J^{K}  \geq 0 $.
In this paper, we started from the value-based perspective to derive monotonic improvement formulation and provide scalable implementation suitable for learning with deep networks.

We verified that CPP can approximately ensure monotonic improvement in low-dimensional problems and achieved superior tradeoff between learning speed and stabilized learning in high-dimensional Atari games.
This tradeoff is best seen from the value of $\zeta$: 
in the beginning of learning the agent prefers to be cautious, resulting in small $\zeta$ values as can be seen from Figure (\ref{pendulum_zeta}). 
In relatively simple scenarios where exact computation or linear function approximation suffices, $\zeta < 1$ might slow down convergence rate in favor of more stable learning. 
On the other hand, in challenging problems this cautiousness might in turn accelerate learning in the later stages, as can be seen from the CVI curves in Figure \ref{atari} that correspond to drastically setting $\zeta=1$: except in the environment \texttt{MsPacman}, in all other environments CVI performed worse than CPP.
This might be due to that learning with deep networks involve heavy approximation error and noises. 
Smoothly changing of the interpolation coefficient becomes necessary under these errors and noises, which is a core factor of CPP.
We found that CPP was especially useful in challenging tasks where both learning progress and cautiousness are required.
We believe CPP bridges the gap between theory and practice that long exists in the monotonic improvement RL literature: previous algorithms have only been tested on simple environments yet failed to deliver guaranteed stability.

CPP made a step towards practical monotonic improving RL by leveraging entropy-regularized RL. 
However, there is still room for improvement. 
Since the entropy-regularized policies are Boltzmann, generally the policy interpolation step does not yield another Boltzmann by adding two Boltzmann policies. 
{\color{black}
Hence an information projection step should be performed to project the resultant policy back to the Boltzmann class to retrieve Boltzmann properties.
}
While this projection step can be made perfect in the ideal case, in practice there is an unavoidable projection error. 
This error if well controlled, could be damaging and significantly degrade the performance. How to remove this error is an interesting future direction.  
% However, it is currently unknown whether CVI guarantees continue to hold for the projected policy, we expand on this issue in the next paragraph.

Another sublety of CPP is on the use of Lemma \ref{thm:KL}.
Lemma \ref{thm:KL} states that the maximum KL divergence of a sequence of CVI policies is bounded. 
However, since we performed interpolation on top of CVI policies, it is hence not clear whether this guarantee continues to hold for the interpolated policy, which renders our use of Lemma \ref{thm:KL} heurisitic.
As demonstrated by the experimental results, we found such heuristic worked well for the problems studied. 
We leave the theoretical justification of Lemma \ref{thm:KL} on interpolated policies to future work.
% However, currently we can neither prove it nor disprove it.
% This is because CVI policies are Boltzmann functions of value functions.
% In principle, the value functions might take on any values due to error/noise/initial conditions.
% This implies that if the interpolated policy happens to possess the same value as arbitrary CVI policy, Lemma \ref{thm:KL} can still hold. 
% But we point out that our use of Lemma \ref{thm:KL} was motivated by this theory but radically heuristic.

We believe the application of CPP, i.e., the combination of policy interpolation and entropy-regularization to other state-of-the-art methods is feasible at least within the value iteration scenario. 
Indeed, CPP performs two regularization: one in the stochastic policy space and the other in the reward function. 
There are many algorithms share the reward function regularization idea with CPP, which implies the possibility of adding another layer of regularization on top of it. 
On the other hand, distributional RL methods may also benefit from the interpolation since they output distribution of rewards which renders interpolation straightforward. We leave them to future investigation.

Another interesting future direction is to extend CPP to the actor-critic setting that can handle continuous action spaces.
Though both CPI-based and entropy-regularized concepts have been respectively applied in actor-critic algorithms, there has not seen published results showing featuring this combination. 
We expect that the combination could greatly alleviate the policy oscillation phenomenon in complicated continuous action control domain such as Mujoco environments.
%Our future work includes applying the proposed algorithms on problems with higher dimensional state spaces with nonlinear function approximators such as deep networks~\cite{Vieillard-2020DCPI}. For this, several theoretical points require further consideration such as changing the on-policy nature of the CPI and SPI to off-policy to fully leverage the merits of deep RL such as an experience replay technique.

%For future work, one interesting future direction would be to extend the current interpolation scheme from consecutive policies to any policies in a sequence. To this end, a number of technical difficulties should be addressed, e.g., proving that the Lemma \ref{thm:boundedStat} still applies to the resultant policy produced by interpolating several policies in a sequence.

\section{Conclusion}\label{sec:conclusion}

In this paper we proposed a novel RL algorithm: cautious policy programming that leveraged a novel entropy regularization aware lower bound for monotonic policy improvement. 
The key ingredients of the CPP is the seminal policy interpolation and entropy-regularized policies.
Based on this combination, we proposed a genre of novel RL algorithms that can effectively trade off learning speed and stability, especially inhibiting the policy oscillation problem that arises frequently in RL applications. 
We demonstrated the effectiveness of CPP against existing state-of-the-art algorithms on simple to challenging environments, in which CPP achieved performance consistent with the theory.

\section{Acknowledgement}

This research is funded by JSPS KAKENHI Grant Number 21H03522 and 21J15633.
 %In the latter, comparison with SPI demonstrates that the proposed algorithm is especially suitable for large state spaces.

%% The file named.bst is a bibliography style file for BibTeX 0.99c
%\bibliographystyle{named}
%\bibliography{library}

\appendix\label{apdx}
%In this appendix we prove Corollary 3 and Theorem \ref{thm:main}. In a big picture, Corollary 3 shows how the guaranteed improvement relates to the value of $\zeta$. Theorem \ref{thm:main} summarizes the results and provide the final algorithm.

\section{Appendix}\label{appendix}

In the first part of the Appendix, we detail the proofs of the theorems and lemmas that appear in our paper. We provide implementation details in the latter half. %In the first part, two important lemmas from recent literature are included and brief proof are provided for completeness.

\subsection{Proof of Theorem \ref{thm:main}}\label{thm1}
In order to prove Theorem \ref{thm:main}, we introduce the following two lemmas. 
The first concerns monotonic policy improvement and the second provides a tool for connecting it with the entropy-regularization-aware lower bound.

\subsubsection{Monotonic Policy Improvement Lemma}\label{apdx:lemma2} 
In this section we provide the proof of Lemma \ref{thm:SPI}.
The proof was borrowed from \cite{pirotta13}
but for the ease of reading we rephrase it here.

%Following this concept, \cite{pirotta13} proposes to optimize the coefficient $\zeta$ to attain a maximum lower bound $\Delta J^{\pi'}_{\pi, d^{\pi'}}$ on policy improvement. The optimal value $\zeta^{*}$ hence represents the optimal policy of a linear policy class spanned by $\tilde{\pi}$ and $\pi$. The following lemma relates the lower bound on improvement to the maximum total variation of $\tilde{\pi}, \pi$:

\textbf{Lemma \ref{thm:SPI}.} \emph{Provided that policy $\pi'$ is generated by partial update Eq. (\ref{mixture_policy}), $\zeta$ is chosen properly, and $A_{\pi, d^{\pi}}^{\tilde{\pi}} \geq 0$, then the following improvement is guaranteed:} 
\begin{align}
  \begin{split}
&\Delta J^{\pi'}_{\pi, d^{\pi'}} \geq \frac{\big((1-\gamma)A_{\pi,d^{\pi}}^{\tilde{\pi}}\big)^{2}}{2\gamma\delta\Delta A^{\tilde{\pi}}_{\pi}}, \\
\text{with } & \zeta = \min{(1, \zeta^{*})},\\
\text{where } &\zeta^{*}=\frac{(1-\gamma)^{2}A^{\tilde{\pi}}_{{\pi, d^{\pi}}}}{\gamma\delta\Delta A^{\tilde{\pi}}_{\pi}},\\
&\delta=\max_{s}{\left|\sum_{a\in\mathcal{A}}\big(\tilde{\pi}(a|s)-\pi(a|s)\big)\right|},\\
&\Delta A^{\tilde{\pi}}_{\pi}=\max_{s, s'}{|A^{\tilde{\pi}}_{\pi}(s)-A^{\tilde{\pi}}_{\pi}(s')}|.
  \end{split}
  \label{J_first_exact_apdx}
\end{align}

\begin{proof}

The proof follows the similar derivation in the classic CPI \cite{Kakade02} and similar results appeared many times in e.g. \cite{pirotta13,Metelli18-configurable}.
We also show that the role of $\zeta$ and $(1-\zeta)$ in Eq. (\ref{mixture_policy}) can be exchanged by solving a similar problem.
To begin, we leverage Theorem 3.5 of \cite{pirotta13} that:
\begin{align}
  \begin{split}
    \Delta J^{\pi'}_{\pi, d^{\pi'}} &\geq A_{\pi,d^{\pi}}^{{\pi'}} - \frac{\gamma\Delta A_{\pi}^{{\pi'}}}{2(1-\gamma)^{2}}\max_{s}{\left|\sum_{a\in\mathcal{A}}\big({\pi'}(a|s)-\pi(a|s)\big)\right|}.
  \end{split}
  \label{pirotta}
\end{align}
Substituting in $\pi' = \zeta \tilde{\pi} + (1-\zeta)\pi$, we have:
\begin{equation}
  \begin{split}
    A^{\pi'}_{\pi,d^{\pi}}&=\sum_{s}{d^{\pi}{(s)}\sum_{a}{\pi'(a|s)A_{\pi}(s,a)}}\\
    &=\sum_{s}{d^{\pi}{(s)}\sum_{a}{\big(\zeta\tilde{\pi}(a|s)+(1-\zeta)\pi(a|s)\big)A_{\pi}(s,a)}}\\
    & =\zeta\sum_{s}{d^{\pi}{(s)}\sum_{a}{\tilde{\pi}(a|s)A_{\pi}(s,a)}}=\zeta A^{\tilde{\pi}}_{\pi,d^{\pi}},
      \end{split}
  \label{SPI_loosen}
\end{equation}
\begin{equation}
    \begin{split}
    \Delta A^{{\pi'}}_{\pi}&=\max_{s, s'}{|A^{{\pi'}}_{\pi}(s)-A^{{\pi'}}_{\pi}(s')}|\\
    &= \max_{s, s'}{|\zeta A^{\tilde{\pi}}_{\pi}(s) -\zeta A^{\tilde{\pi}}_{\pi}(s')}|,\\
    \delta &= \max_{s}{\left|\sum_{a\in\mathcal{A}}\big({\pi'}(a|s)-\pi(a|s)\big)\right|},\\
    &= \max_{s}{\left|\sum_{a\in\mathcal{A}}\big(\zeta\tilde{\pi}(a|s)-\zeta\pi(a|s)\big)\right|}.\\
  \end{split}
  \label{SPI_loosen2}
\end{equation}
Hence, Eq. (\ref{pirotta}) is transformed into:
\begin{align}
  \begin{split}
    \Delta J^{\pi'}_{\pi, d^{\pi'}} &\geq \zeta A_{\pi,d^{\pi}}^{{\tilde{\pi}}} - \frac{\gamma\zeta^{2}\Delta A_{\pi}^{{\tilde{\pi}}}}{2(1-\gamma)^{2}}\max_{s}{\big|\sum_{a\in\mathcal{A}}\big({\tilde{\pi}}(a|s)-\pi(a|s)\big)\big|}.
  \end{split}
  \label{SPI_quadratic}
\end{align}
The right hand side (r.h.s.) is a quadratic function in $\zeta$ and has its maximum at 
\begin{align}
  \begin{split}
\zeta^{*} = \frac{(1-\gamma)^2 A_{\pi,d^{\pi}}^{{\tilde{\pi}}}}{\gamma\Delta A_{\pi}^{{\tilde{\pi}}}\max_{s}{\big|\sum_{a\in\mathcal{A}}\big(\tilde{\pi}(a|s)-\pi(a|s)\big)\big|}}.
  \end{split}
\end{align}
By substituting $\zeta^{*}$ back to Eq. (\ref{SPI_quadratic}), we obtain: 
\begin{align}
  \begin{split}
    \Delta J^{\pi'}_{\pi, d^{\pi'}} \geq \frac{\big((1-\gamma)A_{\pi,d^{\pi}}^{\tilde{\pi}}\big)^{2}}{2\gamma\delta\Delta A^{\tilde{\pi}}_{\pi}}.
  \end{split}
  \label{SPI_maximum}
\end{align}
When $\zeta^{*}>1$, we clip it using $\min(1, \zeta^{*})$. 

Note that, if we exchange the roles of $\zeta$ and $(1-\zeta)$, the coefficients in Eq. (\ref{SPI_loosen}) should be $(1-\zeta)$. Equation (\ref{SPI_quadratic}) would become a quadratic function in $(1-\zeta)$; hence the r.h.s. of Eq. (\ref{SPI_maximum}) would be the maximum of $(1-\zeta^{*})$.
This concludes the proof.

\end{proof}

\textbf{Remark. } 
By noting that $\tilde{\pi}(a|s) - \pi(s,a)$ appears in both $\delta$ and $\Delta A^{\tilde{\pi}}_{\pi}$, we see that the policy improvement $\Delta J^{\pi'}_{\pi, d^{\pi}}$ is governed by the maximum total variation of policies. While one can exploit Lemma \ref{thm:SPI} for a value-based RL algorithm, it can be seen that it could only apply to problems with small-state action spaces. In general, without further assumptions on $\pi', \tilde{\pi}, \pi$, lower-bounding policy improvement is intractable, as maximization $\delta$ and $\Delta A^{\tilde{\pi}}_{\pi}$ in a large state space require exponentially many samples for accurate estimation.

\subsubsection{Entropy-regularization Lemma}\label{apdx:lemma3} 
To optimize the lowerbound in Lemma \ref{thm:SPI}, it is required to know $\delta$ \cite{pirotta13}, which is intractable for large state spaces without further specification on the considered policy class. 

{\color{black}
{
By considering the class of entropy-regularized MDPs, Lem\-ma 2 can be significantly simplified, of which the following lemma plays a crucial role.

%Our approach is based on the aforementioned entropy-regularized value-based algorithms that have achieved state-of-the-art performance on several benchmark problems \cite{haarnoja-SAC2018,ZHU2020CEP}. A very recent study of which offers a means to bound the maximum distance between pre- and post-update policies \cite{kozunoCVI}. The key insight of our approach is that by considering the class of entropy-regularized policies, Lemma \ref{thm:SPI} can be significantly simplified to apply to large state spaces. To begin with our derivation, we first introduce the following lemma:

%Our starting point is to avoid the intractable maximization problems $\Delta A^{\tilde{\pi}}_{\pi}$ and $\delta$ in Eq. (\ref{J_first_exact}) so to extend their use to continuous and high dimensional cases. By noticing that they appeared only in the denominator, one is allowed to lower-bound $\Delta J^{\pi'}_{\pi, d^{\pi'}}$ by upper-bounding $\Delta A^{\tilde{\pi}}_{\pi}$ and $\delta$, respectively. In order for a useful algorithm, it is desired that they be tightly bounded.

\textbf{Lemma \ref{thm:KL}.} \emph{For any policies $\pi_{K}$ and $\pi_{K+1}$ generated by taking the maximizer of Eq. (\ref{sys_DPPbellman}), the following bound holds for their maximum total variation}:
\begin{align}
  \begin{split}
    &\max_{s}{D_{TV}\left(\pi_{K+1}(\cdot|s) \,||\, \pi_{K}(\cdot|s) \right) } \leq \\
    & \qquad \qquad \qquad \qquad \quad \min \left\{  \sqrt{1 - e^{- 4 B_{K} - 2 C_{K}}}, \sqrt{8 B_{K} +  4 C_{K}} \right\} , \\
    & \qquad \text{where } B_{K}=\frac{1-\gamma^{K}}{1-\gamma}\epsilon\beta , \,\,\,\, C_{K} = \beta r_{max} \sum_{k=0}^{K-1}{\alpha^{k}\gamma^{K-k-1}},
  \end{split}
  \label{CVI_kl_apdx1}
\end{align}
\emph{$K$ denotes the current iteration index and $0\leq k\leq K-1$ is the loop index. 
$\epsilon$ is the uniform upper bound of error.} 

\begin{proof}
  By the Fenchel conjugacy of the Shannon entropy and KL divergence \cite{cvx-opt-boyd}, it is clear that the maximizing policies for the regularized MDP are Boltzmann softmax \cite{geist19-regularized} as shown in Section \ref{sec:policyIter}.
  The relationship between Boltzmann softmax policies has recently been actively investigated \cite{azar2012dynamic,asadi17a}.
 We leverage the very recent result \cite[Propsition 3]{kozunoCVI}, which states that:
  \begin{align}
    \begin{split}
      &\max_{s}{D_{KL}\left(\pi_{K+1}(\cdot|s) \,||\, \pi_{K}(\cdot|s) \right) } \leq 4 B_{K} + 2 C_{K},\\
      \text{where } &B_{K}=\frac{1-\gamma^{K}}{1-\gamma}\epsilon\beta , \,\, C_{K} = \beta r_{max} \sum_{k=0}^{K-1}{\alpha^{k}\gamma^{K-k-1}},
    \end{split}
    \label{CVI_kl_apdx2}
  \end{align}
  where $\epsilon$ is the uniform upper bound of errors.

While Pinsker's inequality $D_{TV}(p||q) \leq \sqrt{2 D_{KL}(p||q)}$, where $p,q$ are distributions  can be used to directly exploit Eq. (\ref{CVI_kl_apdx2}), there is a gap between the total variation and KL divergence since $D_{TV} \leq 1$ and $D_{KL}$ is potentially unbounded.
Leveraging Pinsker's inequality on Eq. (\ref{CVI_kl_apdx2}) and then on Eqs. (\ref{SPI_loosen},\ref{SPI_loosen2}) will result in large errors when $D_{KL} \geq \frac{\sqrt{2}}{2}$.

To tackle this problem, we introduce the following bound due to \cite{Bretagnolle-betterTVKL} that has more benign behavior\footnotemark\footnotetext{Eq. (\ref{eq:betterKL} appears in other places in different forms such as in \cite[Eq. (4)]{Sason2016-fDiverInequalities}). It is worth mentioning they are the same in essence and differ only in notations.}:
\begin{align}
  D_{TV}(p || q) \leq \sqrt{ 1 - e^{ - D_{KL}(p || q)} }.
  \label{eq:betterKL}
\end{align}
A similar bound appears also in \cite{tsybakov-nonparametric} but is a slightly looser.
More relevant inequalities of such kind can be found in \cite{Sason2016-fDiverInequalities}.
Both \cite{Bretagnolle-betterTVKL} and \cite{tsybakov-nonparametric} feature the component $e^{-D_{KL}(p||q)}$ that ensures the total variation bound is well-defined: the upperbound $\sqrt{ 1 - e^{ - D_{KL}(p || q)} }$ is guaranteed to be no large than 1. Hence we can combine Eq. (\ref{eq:betterKL}) with Eq. (\ref{CVI_kl_apdx2}) by taking the maximization on both sides, yielding the following relationship:
\begin{align}
  \begin{split}
   \max_{s} D_{TV}(\pi_{K+1}(\cdot | s) || \pi_{K} (\cdot | s)) & \leq \sqrt{ 1 - e^{ - \max_{s}D_{KL}(\pi_{K+1}(\cdot | s) || \pi_{K} (\cdot | s) } } \\
  & \leq \sqrt{1 - e^{ - 4 B_{K} - 2 C_{K}} }.
\end{split}
\label{eq:tv_exp}
\end{align}
Now by applying Pinsker's inequality on Eq. (\ref{CVI_kl_apdx2}), we have the following relationship:
\begin{align}
  \max_{s}{D_{TV}\left(\pi_{K+1}(\cdot|s) \,||\, \pi_{K}(\cdot|s) \right) } \leq \sqrt{ 8 B_{K} + 4 C_{K}},
  \label{eq:tv_bc}
\end{align}
taking the minimum of Eqs. (\ref{eq:tv_exp}, \ref{eq:tv_bc}) yields the promised result.

\end{proof}

}}

Now back to Eq. (\ref{CVI_kl_apdx2}), since the reward is bounded in $[-1, 1]$, $r_{max}$ can be conveniently dropped. 
Also, note that for simplicity we assume there is no update error, i.e., $B_{K}=0$. 
However, it can be straightforwardly extended to cases where errors present by simply choosing an upper-bound $\epsilon$ for errors. 
It is worth noting that in deep RL setting the magnitude of $\epsilon$ might be non-trivial and has to be considered in parameter tuning.
Intuitively, Lemma \ref{thm:KL} ensures that an updated entropy-regularized policy will not deviate much from the previous policy.

\subsubsection{Proof of Theorem \ref{thm:main}}\label{apdx:thm4} 
Now, given Lemma \ref{thm:SPI} and Lemma \ref{thm:KL}, we are ready to prove Theorem \ref{thm:main}. 
We first restate it for ease of reading.

\textbf{Theorem \ref{thm:main}.} \emph{Provided that partial update Eq. (\ref{mixture_cvi}) is adopted,  $A^{{\pi_{K+1}}}_{\pi_{K},d^{\pi_{K}}} \geq 0$, and $\zeta$ is chosen properly, then any maximizer policy of Eq. (\ref{sys_DPPbellman}) guarantees the following improvement that depends only on $\alpha, \beta, \gamma \text{ and } A^{{\pi_{K+1}}}_{\pi_{K},d^{\pi_{K}}}$ after any policy update:}

\begin{align*}
  \Delta J^{\tilde{\pi}_{K+1}}_{\pi_{K},d^{\tilde{\pi}_{K+1}}}  &\geq \frac{\big(1-\gamma)^{3}(A_{\pi_{K},d^{\pi_{K}}}^{{\pi_{K+1}}})^{2}}{4\gamma} \max \left\{\frac{1}{1-e^{-2 C_{K}}} \,\, , \,\, \frac{1}{4 C_{K} } \right\},\\
  \text{with } \zeta &= \min{(1, \zeta^{*})}, \quad C_{K} = \beta\sum_{k=0}^{K-1}{\alpha^{k}\gamma^{K-k-1}}, \\
\text{where } \zeta^{*} &= \frac{(1-\gamma)^{3}A^{{\pi_{K+1}}}_{{\pi_{K},d^{\pi_{K}}}}}{2 \gamma } \max\left\{  \frac{1}{1-e^{-2 C_{K}}}, \frac{1}{4 C_{K}} \right\}.
%&= \frac{\gamma\beta(\gamma^{K} - \alpha^{K})}{\gamma - \alpha}.
\end{align*}

\begin{proof}
%\subsection{Bounding Policy Update by Entropy Regularization}
 %However, itself alone is not sufficient to produce a theoretically sound algorithm. %It is worth mentioning that the bound in Lemma \ref{thm:kakade} is tight as shown by \cite{kozunoCVI}. 
  
The proof follows similarly to the proof of Lemma 2 and hence \cite{pirotta13}.
We prove Theorem \ref{thm:main} by noticing the following inequalities hold for $\delta$ and $\Delta A^{\tilde{\pi}}_{\pi}$ of Eq. (\ref{J_first_exact}), respectively:
\begin{align}
  \begin{split}
    &\Delta A^{\tilde{\pi}}_{\pi} = \max_{s, s'}{|A^{\tilde{\pi}}_{\pi}(s)-A^{\tilde{\pi}}_{\pi}(s')}|\\
    & \leq 2\max_{s}{|A^{\tilde{\pi}}_{\pi}(s)|}=2\max_{s}{\big|\sum_{a}{\tilde{\pi}(a|s) \big(Q_{\pi}(s,a)} - V_{\pi}(s)\big)\big|}\\
    & = 2\max_{s}{\big|\sum_{a}{\big(\tilde{\pi}(a|s)Q_{\pi}(s,a)-\pi(a|s)Q_{\pi}(s,a)\big)}\big|} \\
    & \leq 2 \max_{s}{\sum_{a}{\big|\big(\tilde{\pi}(a|s)-\pi(a|s)\big)Q_{\pi}(s,a)\big|}}\\
    &\leq  2\big|\big| Q_{\pi} \big|\big|_{\infty}\max_{s}{\sum_{a}{\big|\tilde{\pi}(a|s)-\pi(a|s)\big|}} \\
     & \leq 2\sqrt{2}V_{max} \max_{s}{\sqrt{D_{KL}\big(\tilde{\pi}(\cdot|s)||\pi(\cdot|s)\big)}},
  \end{split}
  \label{J_exact_improved_apdx}
\end{align}
 where $V_{max}:=\frac{1}{1-\gamma}r_{max}$ is the maximum possible value function.
 Since we assume reward is upper bounded by $1$, $V_{max}=\frac{1}{1-\gamma}$.
The second inequality makes use of the triangle inequality:
  \begin{align}
    \delta \leq \max_{s}\sum_{a\in\mathcal{A}}\big|\big(\tilde{\pi}(a|s)-\pi(a|s)\big)\big|,
    \label{triangle}
  \end{align}
and the third inequality makes use of Hölder's inequality $\frac{1}{p}+\frac{1}{q}=1$, with $p$ set to $1$ and $q$ set to $\infty$. The last inequality is because of Pinsker's inequality:
\begin{eqnarray}
\begin{aligned}
  &\max_{s}\sum_{a\in\mathcal{A}}\big|\tilde{\pi}(a|s)-\pi(a|s)\big|\leq \max_{s}{\sqrt{2D_{{KL}}(\tilde{\pi}(\cdot|s)||\pi(\cdot|s))}},
  \label{pinsker}
\end{aligned}
\end{eqnarray}
and the fact that $||Q_{\pi}||_{\infty}\leq V_{max}=\frac{1}{1-\gamma}$.

Following \cite{pirotta13}, by incorporating Eqs. (\ref{triangle}, \ref{pinsker}) and Eqs. (\ref{eq:tv_exp}, \ref{eq:tv_bc}) into $\Delta J^{\pi'}_{\pi, d^{\pi'}} \geq \frac{\big((1-\gamma)A_{\pi,d^{\pi}}^{\tilde{\pi}}\big)^{2}}{2\gamma\delta\Delta A^{\tilde{\pi}}_{\pi}}, $ in Eq. (\ref{J_first_exact_apdx}) we have:

\begin{align}
  \begin{split}
    \Delta J^{\pi'}_{\pi, d^{\pi'}} &\geq \frac{\left((1-\gamma) A^{\tilde{\pi}}_{\pi, d^{\pi}} \right)^{2}}{2\gamma{\delta\Delta A^{\tilde{\pi}}_{\pi}}} \\
    &\geq \frac{\left((1-\gamma) A^{\tilde{\pi}}_{\pi, d^{\pi}} \right)^{2}}{2\gamma} \frac{1}{\underbrace{ \max_{s}D_{TV}}_{\delta, \,\,\, \text{Eq.(A.15)}}   } \frac{1}{\underbrace{ 2 V_{max} \max_{s} D_{TV}}_{\Delta A^{\tilde{\pi}}_{\pi}, \,\,\, \text{Eq.(A.13)} } } \\
    &= \frac{(1-\gamma)^3 (A^{\tilde{\pi}}_{\pi, d^{\pi}})^{2} }{ 2\gamma} \frac{1}{2 \max_{s}D_{TV}^{2}} \\
    & \stackrel{\text{(A.9)}}{\geq} \frac{(1-\gamma)^3 (A^{\tilde{\pi}}_{\pi, d^{\pi}})^{2} }{ 4\gamma} \frac{1}{{4 C_{K}} }, \,\,\, \\
    \text{ or } \,\,\,  &\stackrel{\text{(A.11)}}{\geq} \frac{(1-\gamma)^3 (A^{\tilde{\pi}}_{\pi, d^{\pi}})^{2} }{ 4\gamma} \frac{1}{1 - e^{-2C_K} }, \\
  \end{split}
  \label{proof_leq}
\end{align}
by taking the maximum of the two possible outcomes, the result becomes:
\begin{align*}
    \Delta J^{\pi'}_{\pi, d^{\pi'}} &\geq {\color{black}{\frac{(1-\gamma)^{3}}{4\gamma} }} \cdot (A^{\tilde{\pi}}_{\pi, d^{{\pi}}})^2 \cdot \max \left\{ \frac{1}{1 - e^{-2C_{K}}} ,   \frac{1}{4 C_{K}}\right\}.
\end{align*}
%By noting that the bound of Eq. (\ref{proof_leq}) can be extended using Eq. (\ref{CVI_kl}), we obtain our practical algorithm depending only on the tunable parameters $\alpha, \beta$, and $\gamma$:
The way of choosing $\zeta$ is same as Eq. (\ref{SPI_maximum}) solving the equation that is negative quadratic in $\zeta$. 

\end{proof}

\subsection{Implementation Details}\label{apdx:atari}

In Algorithm \ref{alg:deepCPP} we followed \cite{Vieillard-2020DCPI} for computing the stationary weighted advantage function that empirically shows good performance. 
It should be noted that accurately estimating stationary distribution is still nontrivial \cite{wen2020-batchStationryEstimation} and we leave the improvement to CPP in this regard to our future work. 
% Another important point is Line 8 in Algorithm \ref{alg:deepCPP} which corresponds to computing moving average of advantage function. 
% The intuition here is to collect more data in order to lower the variance so there is more chance that the expected advantage becomes positive. 

Deep CPP, MDQN, MoDQN and CVI in the experimental section share the same network architecture and hyperparameters as specified by the following table:
\begin{center}
  \begin{tabular}{ | c | c | } 
  \hline
  \textbf{Hyperparameters} &  \textbf{Values} \\ 
  \hline
  Number of convolutional layers &  3 \\ 
  \hline
  Convolutional layer channels & (32, 64, 64) \\ 
  \hline
  Convolutional layer kernel size & (8, 4, 3)\\
  \hline
  Convolutional layer stride & (4, 2, 1)\\
  \hline
  Number of fully connected layers &  1, with 512 hidden units \\ 
  \hline
  Batch size & 64\\
  \hline
  Replay buffer size & $10^{6}$ \\
  \hline
  Discount rate & 0.99 \\
  \hline
  Steps per update & 4\\
  \hline
  Learning rate & $1 \times 10^{-4}$\\
  \hline
  Optimizer & Adam \cite{Adam}\\
  \hline
  Loss & Mean squared error \\
  \hline
  $T$ the total number of steps & $5 \times 10^6$\\
  \hline
  $F$ the interaction period & $4$\\
  \hline
  $C$ the update period & $8000$\\
\hline
  \end{tabular}
  \end{center}

  By comparing our results on Deep CPP and Deep CPI \cite{Vieillard-2020DCPI} we see there is difference on the horizon. We ran all algorithms for $5 \times 10^6$ steps while Deep CPI was ran for $5 \times 10^7$ steps. However, we can still make a comparison by the scores up to $5 \times 10^6$ steps. By comparing on the environments that appeared in both papers we have:
  \begin{center}
      \begin{tabular}{ c|c|c } 
      \hline
        Environment      & DCPP & DCPI \\
      \hline
      MsPacman & 2000 & 2200 \\ 
      \hline
      SpaceInvaders & 800 & 800 \\
      \hline
      Seaquest & 3000 & 2000 \\
      \hline
      \end{tabular}
  \end{center} 
Hence we see on relatively simple environments like \texttt{MsPacman} and \texttt{SpaceInvaders} DCPP and DCPI performed similarly. On the other hand, on the challenging environment Seaquest \cite{azar18-noisynet}, DCPP achieved around $30\%$ higher scores at the end of $5 \times 10^6$ environment steps.

We also report the tuned hyperparameters unique to each algorithm in Figure 3 using Optuna \cite{optuna}:
\begin{center}
  \begin{tabular}{ c | c }
    \hline
    Algorithm & Parameters \\
    \hline 
    \multirow{2}*{CPP} &  \multirow{2}*{\shortstack{ Entropy $\tau$: 0.0124 \\ KL regularization $\sigma$: 0.001}} \\
    &  \\
    \hline
    \multirow{2}*{MDQN} &  \multirow{2}*{\shortstack{ Entropy $\tau$: 0.03 \\ Munchausen term $\sigma$: 0.9}} \\
    &  \\
    \hline
    \multirow{2}*{MoDQN} &  \multirow{2}*{Same with \cite{Vieillard2020Momentum}} \\
    &  \\
    \hline
    \multirow{2}*{CVI} &  \multirow{2}*{\shortstack{ Gap coefficient $\alpha$: 0.00024 \\ Temperature $\beta$: 0.000225}} \\
    &  \\
    \hline
  \end{tabular}
\end{center}
The hyperparameters were obtained by running on the environment \texttt{SpaceInvaders} for 300 Optuna trials \cite{optuna}. Each trial consists of $10^5$ steps and the resultant 300 sets of parameters were ranked.
For the on-policy algorithms, PPO and A2C are built-in with Stable Baselines 3 library \cite{stable-baselines3} and the parameters were already fine-tuned. We evaluated them without changing their default hyperparameters.

\bibliography{../../library}

\end{document}